\documentclass{article}
\usepackage[T1]{fontenc}
\usepackage{iclr2027_conference,times}

\usepackage{amsmath,amsfonts,bm}

\def\eqref#1{equation~\ref{#1}}

\def\1{\bm{1}}

\DeclareMathAlphabet{\mathsfit}{\encodingdefault}{\sfdefault}{m}{sl}
\SetMathAlphabet{\mathsfit}{bold}{\encodingdefault}{\sfdefault}{bx}{n}

\usepackage{microtype}
\usepackage{amsmath,amssymb,amsthm,mathtools,bm}
\usepackage{enumitem}
\usepackage{booktabs,multirow,graphicx,subcaption}
\usepackage{pgfplots}
\pgfplotsset{compat=1.18}
\usepackage[ruled,linesnumbered]{algorithm2e}
\usepackage{float,pifont}
\usepackage{adjustbox}
\IfFormatAtLeastTF{2026-06-01}{}{\usepackage{aliascnt}}
\usepackage[table]{xcolor}

\definecolor{rankone}{HTML}{0072B2}   
\definecolor{ranktwo}{HTML}{D55E00}   
\definecolor{rankthree}{HTML}{009E73} 

\usepackage{hyperref}
\hypersetup{hidelinks,
  pdftitle={When Missingness Follows the Graph: Beyond Feature Propagation for Missing Features},
  pdfauthor={Enrico Chialà, Francesco Ferrini, Antonio Longa, Bruno Lepri, Filippo Maria Bianchi, Manfred Jaeger, Veronica Lachi}}

\usepackage{wrapfig}
\usepackage{url}
\usepackage{cleveref}

\newtheorem{theorem}{Theorem}
\newtheorem{proposition}{Proposition}
\IfFormatAtLeastTF{2026-06-01}{
  \newtheorem{lemma}[theorem]{Lemma}
}{
\newaliascnt{lemma}{theorem}

\aliascntresetthe{lemma}
}
\IfFormatAtLeastTF{2026-06-01}{
  \newtheorem{corollary}[theorem]{Corollary}
}{
\newaliascnt{corollary}{theorem}

\aliascntresetthe{corollary}
}
\theoremstyle{definition}
\IfFormatAtLeastTF{2026-06-01}{
  \newtheorem{definition}[theorem]{Definition}
}{
\newaliascnt{definition}{theorem}
\newtheorem{definition}[definition]{Definition}
\aliascntresetthe{definition}
}
\theoremstyle{remark}
\IfFormatAtLeastTF{2026-06-01}{
  \newtheorem{remark}[theorem]{Remark}
}{
\newaliascnt{remark}{theorem}

\aliascntresetthe{remark}
}
\crefname{definition}{Definition}{Definitions}
\Crefname{definition}{Definition}{Definitions}
\crefname{proposition}{Proposition}{Propositions}
\Crefname{proposition}{Proposition}{Propositions}
\crefname{lemma}{Lemma}{Lemmas}
\Crefname{lemma}{Lemma}{Lemmas}

\newcommand{\PE}{\mathrm{PE}}

\newcommand{\GCNmfPE}{\textsc{PEMix}}

\newcommand{\RTMAR}{\textsc{RT}}

\newcommand{\todo}[1]{}
\newcommand{\MODEL}{\textsc{PEMix}}

\newcommand{\ff}[1]{}

\iclrfinalcopy
\title{When Missingness Follows the Graph: Beyond Feature Propagation for Missing Features}
\author{\begin{minipage}[t]{\dimexpr\textwidth-2\tabcolsep\relax}
\vspace{0.8em}
\centering
{\normalsize\bfseries
Enrico Chialà\textsuperscript{1}, Francesco Ferrini\textsuperscript{2}, Antonio Longa\textsuperscript{3}, Bruno Lepri\textsuperscript{2},\\[0.2em]
Filippo Maria Bianchi\textsuperscript{3}, Manfred Jaeger\textsuperscript{4}, Veronica Lachi\textsuperscript{3}}\\[0.7em]
{\normalfont\small
\textsuperscript{1}University of Trento\quad
\textsuperscript{2}Fondazione Bruno Kessler\\[0.15em]
\textsuperscript{3}UiT, The Arctic University of Norway\quad
\textsuperscript{4}Aalborg University}
\end{minipage}}
\date{}
\begin{document}
\maketitle
\fancyhead{}
\fancyhead[L]{Preprint}
\renewcommand{\headrulewidth}{0.4pt}
\raggedbottom

\begin{abstract}
Graph Neural Networks (GNNs) commonly assume that node features are fully observed, despite missing features being pervasive in real-world applications. 
By leveraging the dependence between node features and graph structure, existing graph-specific methods handle missingness by propagating observed feature values along graph edges.
However, these methods have mainly been evaluated in overly optimistic settings, where feature entries are removed independently and uniformly across the graph. 
This overlooks realistic scenarios in which missingness propagates among connected nodes, as can occur during localized sensing failures in sensor networks. We formalize this complementary notion as \emph{graph-structured mask dependence}, where the joint distribution of missingness indicators is shaped by the graph, and introduce a controlled mechanism that instantiates this setting.
We study propagation-based methods theoretically and empirically in this setting, showing that their effectiveness changes substantially  compared with the uniformly random regime. 
Motivated by this limitation, we propose PEMix, designed to remain effective under challenging graph-structured missingness. Rather than propagating observed values, PEMix uses global structural representations to condition a probabilistic model of incomplete node features before message passing. Experiments under controlled graph-structured missingness and naturally occurring missingness show that PEMix remains robust across complex missingness regimes, with particularly strong performance as missingness becomes severe.

\end{abstract}

\section{Introduction}
\label{sec:intro}
Graph Neural Networks (GNNs) are commonly designed under the assumption that node features are fully observed \citep{kipf2017gcn,hamilton2017graphsage,velickovic2018gat}. 
In practice, this assumption is often violated. 
Sensors may fail to record part of their measurements \citep{braem2024missing,faizin2019review}, while communication, acquisition, or infrastructure failures can also lead to missing observations.

Naive approaches rely on simple statistics to fill missing data, treating the feature matrix as tabular~\citep{ferrini2026rethinking,TAGUCHI2021155}. 
Graph-specific methods instead leverage the graph topology to handle missingness more effectively. 
Most existing methods estimate missing features by propagating the observed values while minimizing a Dirichlet
energy \citep{rossi2022unreasonable}.
Methods such as PCFI and FISF extend this principle with confidence-based and iterative propagation schemes \citep{um2023confidencebased,um2025propagate} and achieve strong empirical performance on downstream node-level tasks.

However, performance is mostly evaluated on an oversimplified family of missingness patterns. 
The standard protocol removes feature entries independently under uniform Missing Completely At Random (MCAR) \citep{rubin1976inference}. 
More recent work has considered feature- and label-dependent mechanisms, including Missing Not At Random (MNAR) settings \citep{ferrini2026rethinking}, but these protocols still focus on how missingness depends on feature values or labels. 
They leave unexplored a different aspect of the problem: how missingness is conditioned by topology.
The latter is particularly relevant in applications where failures propagate locally. 
In sensor networks or smart grids, physical or communication faults often cause entire connected regions of nodes to drop their measured features simultaneously~\citep{marisca2024graph}.
Consequently, missingness indicators violate standard spatial independence assumptions, generating correlated missing patterns across adjacent nodes.
Thus, the same global missingness rate can lead to very different levels of difficulty for graph-based methods depending on how missing entries are distributed over the graph. 
These patterns challenge propagation-based methods: when missing-feature nodes cluster together, there is too little observed information to propagate. We analyze this limitation theoretically and empirically and we propose PEMix, designed to
remain effective under challenging graph-structured missingness.


\begin{figure}[t]
    \centering
    \includegraphics[width=1\linewidth]{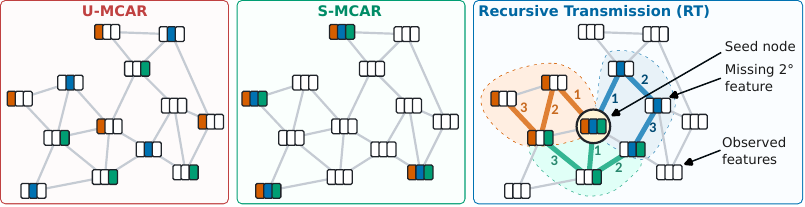}
    \caption{\textbf{Conceptual comparison of missingness mechanisms.} 
    \emph{Left (U-MCAR):} Feature entries are dropped independently and uniformly across the graph. 
    \emph{Middle (S-MCAR):} Entire feature vectors are missing for randomly sampled nodes, independent of connectivity. 
    \emph{Right (Recursive Transmission, RT):} Missingness originates at a shared seed node and expands recursively along graph edges, forming spatially correlated missing regions that challenge propagation-based recovery.}
    \label{fig:placeholder}
\end{figure}

Our contributions are as follows.

\begin{enumerate}[leftmargin=*]

\item We distinguish the classical dependence of the missingness mask on the data from the dependence structure within the mask itself, formalize
\emph{graph-structured mask dependence}, and introduce Recursive Transmission (RT), a controlled graph-structured mechanism.

\item We establish a boundary bottleneck for standard feature propagation and show empirically that graph-localized missingness challenges more sophisticated propagation methods.



\item We propose PEMix to handle clustered missingness where propagation-based approaches fail. It conditions a probabilistic feature model on global structural representations prior to message passing, incorporating graph topology directly without relying on value propagation.

\item  Across controlled graph-structured missingness and real-world datasets, we show that PEMix improves robustness over propagation baselines, achieving particularly strong gains as missingness severity increases.


\end{enumerate}

\section{Learning with Missing Node Features}
\label{sec:learning_missing_features}

In this section, we follow the notation of \citet{ferrini2026rethinking}. We consider a connected attributed graph
$G=(V,E,\mathbf{X},\mathbf{Y})$, where
$V=\{1,\ldots,n\}$ is the set of nodes,
$E\subseteq V\times V$ is the set of edges represented by the
adjacency matrix $\mathbf{A}\in\{0,1\}^{n\times n}$,
$\mathbf{X}\in\mathbb{R}^{n\times d}$ is the node feature matrix,
with $X_{ij}$ denoting feature $j$ of node $i$, and
$\mathbf{Y}\in\{1,\ldots,K\}^n$ is the vector of node labels. Throughout the paper, $\mathbf{A}$ is a random variable in the probabilistic formulation and is fully observed. We condition on its realized value throughout learning and inference.

Let $\mathbf{M}\in\{0,1\}^{n\times d}$ denote the missingness mask,
where $M_{ij}=1$ iff $X_{ij}$ is unobserved. We write
$\mathbf{X}_{\mathrm{obs}}$ and $\mathbf{X}_{\mathrm{mis}}$ for the
observed and missing entries induced by $\mathbf{M}$, respectively.
The learner receives $(\mathbf{A},\mathbf{X}_{\mathrm{obs}},\mathbf{M})$
together with the labels of the training nodes.

The complete-data distribution can be factorized as
\begin{equation}
    p_{\theta,\gamma,\lambda}
    (\mathbf{X},\mathbf{Y},\mathbf{M}\mid\mathbf{A})
    =
    p_\theta(\mathbf{X}\mid\mathbf{A})\,
    p_\gamma(\mathbf{Y}\mid\mathbf{X},\mathbf{A})\,
    p_\lambda(\mathbf{M}\mid\mathbf{X},\mathbf{Y},\mathbf{A}),
    \label{eq:joint_distribution}
\end{equation}
where $p_\theta$ models node features, $p_\gamma$ the complete-data
predictor, and $p_\lambda$ the missingness mechanism. 
Conditioning on $\mathbf{A}$ makes explicit that the fixed graph structure can inform these distributions.

With incomplete node features, the predictive distribution is
\begin{equation}
    p(\mathbf{Y}\mid\mathbf{X}_{\mathrm{obs}},\mathbf{M},\mathbf{A})
    =
    \int
    p(\mathbf{Y}\mid\mathbf{X},\mathbf{M},\mathbf{A})\,
    p(\mathbf{X}_{\mathrm{mis}}
    \mid\mathbf{X}_{\mathrm{obs}},\mathbf{M},\mathbf{A})
    \,d\mathbf{X}_{\mathrm{mis}}.
    \label{eq:predictive_missing}
\end{equation}
A standard GNN parameterizes
$p_\gamma(\mathbf{Y}\mid\mathbf{X},\mathbf{A})$ for complete
$\mathbf{X}$. It is, therefore, necessary to handle incomplete feature
information before or during prediction.

Some methods handle missing node features without using
topological information during feature recovery. They operate only on
the incomplete feature matrix, typically by replacing missing entries in $\mathbf{X}_{\mathrm{mis}}$
with statistics computed from the observed values of the same feature in $\mathbf{X}_{\mathrm{obs}}$,
such as the mean or median. 
For example, GCNmf models the feature distribution with
a Gaussian mixture and uses the resulting conditional moments to handle
missing coordinates \citep{TAGUCHI2021155}. After filling missing
entries with placeholders, GNN-MIM instead concatenates the missingness
mask to the node features \citep{ferrini2026rethinking}. In these
approaches, the graph is still used by the downstream GNN, but
missing features are not recovered using any topological information.

On the other hand, topology-aware approaches such as Feature Propagation (FP)
\citep{rossi2022unreasonable}, estimates missing feature
values by diffusing observed values across graph edges. For a feature
channel $j$, let $\mathbf{x}^{(j)}=\mathbf{X}_{:,j}$,
$O_j=\{i:M_{ij}=0\}$, and
$\mathcal{M}_j=\{i:M_{ij}=1\}$ denote the feature vector and the nodes
where feature $j$ is observed and missing, respectively. 
FP repeatedly updates feature values through
$\widehat{\mathbf{A}}$, the normalized adjacency matrix used for
propagation, while restoring the observed entries after each
step:
\begin{equation}
x_i^{(j,t+1)}
=
\begin{cases}
X_{ij},
    & i\in O_j,\\[2mm]
\displaystyle
\sum_{u\in V}
\widehat{A}_{iu}x_u^{(j,t)},
    & i\in\mathcal{M}_j.
\end{cases}
\label{eq:fp_update}
\end{equation}

PCFI extends this principle by assigning confidence scores based on the
distance from observed features and combining channel-wise diffusion
with inter-channel information \citep{um2023confidencebased}. FISF
alternates propagation with repeated injection of the observed features
\citep{um2025propagate}. Although these methods differ in their
implementation, they all recover missing feature by propagating observed feature values
through graph connections. 

A crucial point often neglected in the literature is that the mechanisms used to describe the missingness are often orthogonal to how missingness is recovered.
The latter is described by the missingness mechanism $p_\lambda$ in~\eqref{eq:joint_distribution}. Following Rubin
\citep{rubin1976inference}, missingness mechanisms can be characterized
through conditional dependencies between the mask and the data. Since
$\mathbf{A}$ is fully observed and fixed in our setting, we express these
conditions through $\mathbf{A}$.

For node classification, a mechanism is \emph{feature-MAR} when
\begin{equation}
    p(\mathbf{M}\mid
    \mathbf{X}_{\mathrm{mis}},
    \mathbf{X}_{\mathrm{obs}},
    \mathbf{A})
    =
    p(\mathbf{M}\mid
    \mathbf{X}_{\mathrm{obs}},
    \mathbf{A}),
    \label{eq:feature_mar}
\end{equation}
i.e., the missingness only depends on the observed features. On the other hand, is
\emph{label-MAR} when
\begin{equation}
    p(\mathbf{M}\mid\mathbf{X},\mathbf{Y},\mathbf{A})
    =
    p(\mathbf{M}\mid\mathbf{X},\mathbf{A}),
    \label{eq:label_mar}
\end{equation}
and \emph{Missing Completely At Random} (MCAR) when
\begin{equation}
    p(\mathbf{M}\mid\mathbf{X},\mathbf{Y},\mathbf{A})
    =
    p(\mathbf{M}\mid\mathbf{A}).
    \label{eq:mcar}
\end{equation}
As in \cite{ferrini2026rethinking} we treat the observed graph structure as fixed throughout learning and inference. Hence, MCAR refers to independence of the missingness mask from the random data variables \(X\) and \(Y\), conditional on the realized graph \(A\).
Since $\mathbf{A}$ is fixed,~\eqref{eq:mcar} is equivalent to
the usual statement that the mask is independent of the data
$(\mathbf{X},\mathbf{Y})$. Missingness is MNAR when the mask depends on hidden feature values or labels. Under feature-MAR and label-MAR, the missingness mechanism is ignorable
for prediction and~\eqref{eq:predictive_missing} reduces to
\begin{equation}
    p(\mathbf{Y}\mid
    \mathbf{X}_{\mathrm{obs}},\mathbf{M},\mathbf{A})
    =
    \int
    p_\gamma(\mathbf{Y}\mid\mathbf{X},\mathbf{A})\,
    p_\theta(\mathbf{X}_{\mathrm{mis}}
    \mid\mathbf{X}_{\mathrm{obs}},\mathbf{A})
    \,d\mathbf{X}_{\mathrm{mis}}.
    \label{eq:ignorable_predictive}
\end{equation}

To assess robustness to missing features, existing methods are commonly evaluated by injecting
missingness into complete graph datasets. The standard
protocol assumes uniform MCAR (U-MCAR), where each feature entry is
removed independently with a fixed probability $\mu$,
$M_{ij}\overset{\mathrm{i.i.d.}}{\sim}\mathrm{Bernoulli}(\mu)$.
Another commonly used MCAR protocol removes the complete feature vectors of uniformly selected nodes; following
\citet{ferrini2026rethinking}, we refer to this setting as Structural
MCAR (S-MCAR).


The MCAR/MAR/MNAR taxonomy characterizes how $\mathbf{M}$ depends on the data, but not how the entries of $\mathbf{M}$ depend on one another. This distinction is crucial for graph-based recovery: two mechanisms can be equally MCAR and share the same global missingness rate, yet arrange missing entries according to the graph topology. This spatial arrangement directly affects propagation-based methods, as the configuration of observed entries determines the information available for diffusion. In the next section, we formalize this structural missingness when the graph shapes the mask.

\section{Graph-Structured Missingness}

To study how graph topology conditions missingness patterns, we first formalize the internal joint organization of the mask $M$. Let $p(\mathbf{M}\mid\mathbf{A})$ denote the joint distribution of the missingness mask conditional on the observed adjacency matrix $\mathbf{A}$. 
We say that the mask is \emph{internally dependent} when its distribution does not factorize over individual entries:


\begin{equation}
    p(\mathbf{M}\mid\mathbf{A}) \neq
    \prod_{i=1}^{n}
    \prod_{j=1}^{d}
    p(\mathbf{M}_{ij}\mid\mathbf{A}).
    \label{eq:mask_dependence}
\end{equation}

Equation~\ref{eq:mcar} constrains the dependence between $\mathbf{M}$ and the
data $(\mathbf{X},\mathbf{Y})$, whereas
Equation~\ref{eq:mask_dependence} concerns dependence among entries
of $\mathbf{M}$: for example, consider an MCAR mechanism that masks exactly
$m$ out of the $N=nd$ feature entries uniformly without replacement. For any two distinct entries $a=ij$ and $b=zk$,
$\Pr(\mathbf{M}_a=1\mid\mathbf{A})=\frac{m}{N}$, while
$\Pr(\mathbf{M}_b=1\mid \mathbf{M}_a=1,\mathbf{A})=\frac{m-1}{N-1}\neq\frac{m}{N}$.
Thus, fixing the global missingness budget already induces dependence among mask entries and is a broad property. In this work, we are interested in a narrower
case in which the graph itself shapes this dependence structure, which is a realistic and challenging scenario. 

\subsection{Graph-Structured Dependence within the Missingness Mask}
We characterize graph-structured dependence at the level of the joint mask distribution.

\begin{definition}[Graph-structured mask dependence]
\label{def:graph_structured_dependence}
Let $p(\mathbf{M}\mid\mathbf{A})$ denote the conditional mask distribution. The missing mechanism exhibits \emph{graph-structured mask
dependence} if

\begin{enumerate}[leftmargin=*]
    \item the internal dependency on the mask is conditional on $\mathbf{A}$ and
    \item the joint mask distribution depends on graph structure, i.e.,
    there exist possible adjacency matrices $\mathbf{A}$ and $\mathbf{A}'$ on the same node set, each with positive probability, such that
    $p(\mathbf{M}\mid\mathbf{A})\neq p(\mathbf{M}\mid\mathbf{A}')$,
    with all non-topological parameters of the mechanism held fixed.
\end{enumerate}
\end{definition}

This definition is broad enough to encompass the missingness mechanisms from the literature: U-MCAR without fixed budget is MCAR
with internally independent mask; fixed-budget U-MCAR induces internally dependent masks without graph-structured dependence.
Finally, S-MCAR also induces strong within-node dependence by masking complete feature vectors, but the
masked nodes are sampled independently of $\mathbf{A}$. These standard protocols do not model graph-structured dependencies among mask entries. In the following, we introduce a mechanism designed to account for such dependencies.

\subsection{Recursive-Transmission MCAR}
\label{sec:rt_mcar}

Graph-structured mask dependence can be instantiated in many ways. 
We propose a mechanism for graph-structured missingness called \emph{Recursive Transmission} (RT) that: i) is directly comparable with
standard MCAR protocols at the same missingness rate; ii) generates a missingness mask that can
challenge existing propagation-based methods; iii) captures a plausible class of localized failure patterns. 
For a target missingness rate $\mu$, let $m=\lfloor \mu nd\rfloor$
be the total number of entries to mask. We distribute this budget as
evenly as possible across each feature channel $j$, obtaining budgets $m_j$
such that $\sum_{j=1}^{d}m_j=m$.
RT starts from a source node $s\in V$ shared across feature channels
and expands each feature-specific missing region only through graph
edges. For feature $j$ with $m_j>0$, initialize $C_j^{(1)}=\{s\}$. At iteration $t$, define
\(
    \partial C_j^{(t)}
    =
    \left\{
        v\notin C_j^{(t)}
        :
        \exists u\in C_j^{(t)}
        \text{ with }(u,v)\in E
    \right\},
\)
and, for each candidate $v\in\partial C_j^{(t)}$,
$a_j^{(t)}(v)=\left|\mathcal{N}(v)\cap C_j^{(t)}\right|$.

The next node is chosen as $v_{t+1}^{(j)}
    \in
    \arg\max_{v\in\partial C_j^{(t)}}
    a_j^{(t)}(v),
    $ with uniform random tie-breaking, independently across feature
channels. We then set $C_j^{(t+1)}=C_j^{(t)}\cup\{v_{t+1}^{(j)}\}$,
and continue until $|C_j|=m_j$. Finally,
$\mathbf{M}_{ij}=\mathbb{I}\{i\in C_j\}$.
The structural differences between missingness protocols are illustrated in Figure~\ref{fig:placeholder}. 
While U-MCAR removes feature entries uniformly at random and S-MCAR drops complete feature vectors for independently selected nodes, RT models localized failures that originate from a shared source node and expand along graph edges, generating compact, spatially correlated missing regions.
The construction is motivated by the three goals above. First, RT
represents a localized failure with a common origin, as may occur in
sensor, communication, power-grid, or acquisition systems, while
feature-specific growth allows different measurements to be affected along different trajectories. 

Among the possible instantiations of the mechanism, we focus on the challenging case in which the seeds of the missingness of each feature coincide with the same source node. 
In general, in RT one could assign a distinct seed $s$ to each feature. 
However, on large graphs, especially at moderate missingness rates and with relatively few features, this would place feature-specific missing regions in different parts of the graph. 
Many affected nodes would then lose only one or a few measurements, while retaining most of their feature vector.
While such masks still contain graph-structured dependence, they describe a missingness that is easier to recover for methods that can exploit nearby observations or even just information from the other feature channels. 
On the other hand, by using the same seed for all channels, RT creates partially overlapping missing regions and therefore areas in which several measurements are simultaneously unavailable. 
Independent tie-breaking still allows the regions to diverge, avoiding to always end up in the more restrictive case where the same complete
set of features is missing at every affected node.

Finally, using a single global rate $\mu$ makes the severity of RT
directly comparable with fixed-budget U-MCAR and S-MCAR. The mechanisms can therefore be matched in the total number of missing entries while differing in how those entries are arranged over the graph.
RT is not intended as a generative model of all real-world missingness patterns. Rather, it is a controlled stress test that isolates the effect of graph-localized mask dependence while matching standard protocols in the total missingness budget. In our experimental setting, we therefore complement RT with experiments on naturally occurring missingness, for which no graph-structured mechanism is imposed.

\begin{proposition}[RT is graph-structured MCAR]
\label{prop:rt_graph_structured_mcar}
If source selection and tie-breaking are independent of $\mathbf{X}$ and $\mathbf{Y}$ conditional on $\mathbf{A}$,
then RT is MCAR and satisfies
Definition~\ref{def:graph_structured_dependence}.
\end{proposition}

A formal proof is provided in
Appendix~\ref{app:proofs}. RT is thus a realistic graph-structured MCAR mechanism whose particular geometry
is chosen to probe a failure mode of propagation-based methods. We
make this connection precise next.

\subsection{Why Graph-Structured Missingness Is Difficult for Propagation}
\label{sec:propagation_bottleneck}

We now show why RT can be particularly challenging for methods that
recover missing features through graph propagation. Consider one feature channel and omit the feature index for readability.
Let $O$ and $M$ denote the observed and missing nodes for that feature,
respectively. The FP update introduced in
Equation~\ref{eq:fp_update} can be written in block form as
\begin{equation}
    x_M^{(t+1)}
    =
    \widehat{\mathbf{A}}_{MM}x_M^{(t)}
    +
    \widehat{\mathbf{A}}_{MO}x_O.
    \label{eq:fp_block}
\end{equation}
The first term propagates the current estimates within the missing set,
while the second introduces observed values into it. Define the observed boundary of the missing region as
\[
    \partial_O M
    =
    \left\{
        u\in O :
        \exists v\in M
        \text{ with }(u,v)\in E
    \right\}.
\]
Only nodes in $\partial_O M$ can directly inject observed information
into the missing region.

\begin{proposition}[Boundary bottleneck for feature propagation]
\label{prop:boundary_bottleneck}
Let $\rho(\widehat{\mathbf{A}}_{MM})$ denote the spectral radius of
$\widehat{\mathbf{A}}_{MM}$ and assume
$\rho(\widehat{\mathbf{A}}_{MM})<1$. Then the iteration in
~\eqref{eq:fp_block} converges to
\begin{equation}
    \widehat{x}_M
    =
    \mathbf{B}x_O,
    \qquad
    \mathbf{B}
    =
    \left(
        \mathbf{I}-\widehat{\mathbf{A}}_{MM}
    \right)^{-1}
    \widehat{\mathbf{A}}_{MO}.
    \label{eq:fp_limit}
\end{equation}
Moreover,
\begin{equation}
    \operatorname{rank}(\mathbf{B})
    \leq
    |\partial_O M|.
    \label{eq:boundary_rank}
\end{equation}
\end{proposition}

The proof is given in Appendix~\ref{app:proofs}.
Equation~\ref{eq:boundary_rank} makes the relevant property of the
mask explicit. When many missing nodes form a compact connected region with a small observed boundary, all reconstructed values inside the
region are determined by only a limited number of observed inputs. In contrast, the same number of missing entries scattered across the graph
may be adjacent to many more observed nodes.

To ground the theory of RT, we consider TADPOLE, which was identified in~\cite{ferrini2026rethinking} as a suitable benchmark
for evaluating robustness to missing node features. We compare RT with U-MCAR and S-MCAR, all with the same global missingness rate.
S-MCAR controls for strong within-node co-missingness while remaining independent of graph
connectivity.
\begin{figure}[t]
    \centering
    \includegraphics[width=\linewidth]{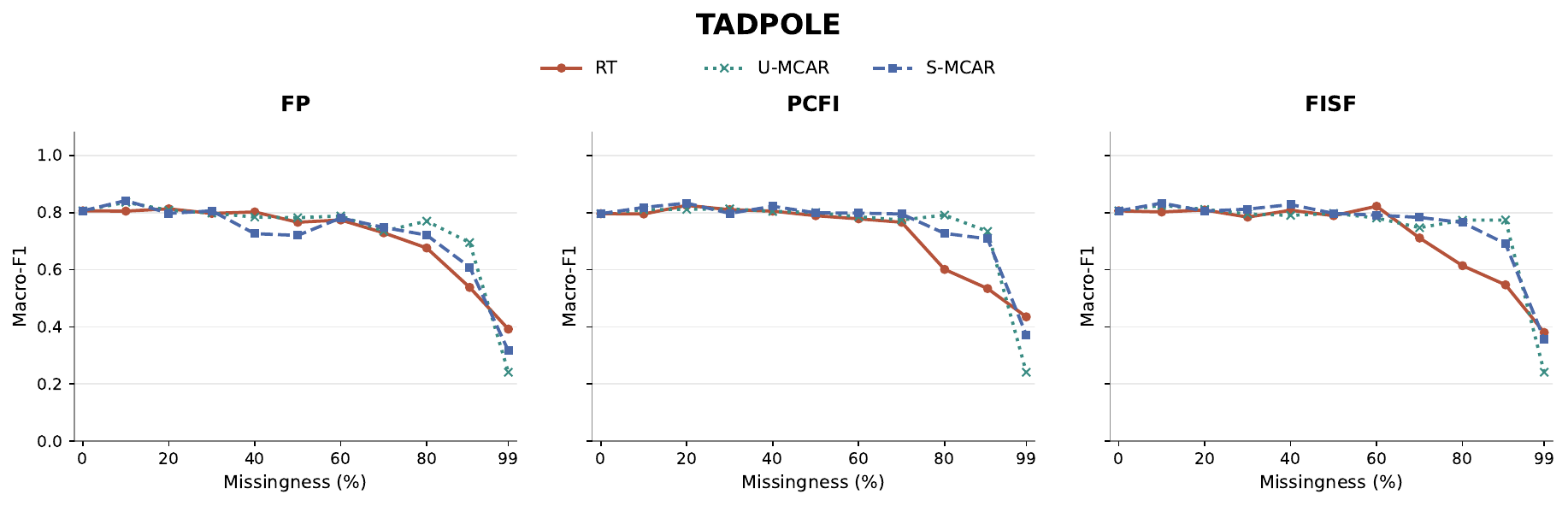}
    \caption{TADPOLE macro-F1 under RT, U-MCAR, and S-MCAR
    for FP, PCFI, and FISF. Curves show the mean and shaded
    bands indicate one standard deviation over five seeds.
    Complete macro-F1 and AUROC results are reported in
    Appendix~\ref{app:propagation-controls}.}
    \label{fig:rt_umcar_smcar}
\end{figure}
Figure~\ref{fig:rt_umcar_smcar} shows that RT is more challenging for propagation-based methods. 
Averaging the macro-F1 AUC across
FP, PCFI, and FISF gives $0.725$ under RT, compared with $0.754$ under
U-MCAR and $0.751$ under S-MCAR. Thus, graph-structured missingness
reduces average propagation performance by approximately $~4\%$
relative to U-MCAR and $~3.5\%$ relative to S-MCAR.

\section{PEMix: Structure-Conditioned Probabilistic Imputation}
\label{sec:pemix}

The analysis in Section~\ref{sec:propagation_bottleneck} shows that
propagation-based methods can become less effective when missingness propagates
through graph connections. The graph, however, remains fully observed and can
still provide information about node attributes. PEMix exploits this
information without using topology to propagate observed feature values into
missing regions.

Let $\mathbf P=g(\mathbf A)\in\mathbb R^{n\times q}$ be a fixed structural
representation of the graph, with $\mathbf p_i$ denoting the representation of
node $i$. In our experiments, $g$ returns standardized non-trivial Laplacian
eigenvectors. PEMix uses $\mathbf P$ to condition a probabilistic model of the
incomplete node features,
$p_\theta(\mathbf X^{\mathrm{mis}}\mid\mathbf X^{\mathrm{obs}},\mathbf A)
\approx
q_\eta(\mathbf X^{\mathrm{mis}}\mid\mathbf X^{\mathrm{obs}},\mathbf P)$.
The positional encodings are computed once from the fully observed graph and
are not learned by the mixture.

\paragraph{Structure-Conditioned Mixture}

For each node $i$, let $\mathbf x_i=\mathbf X_{i,:}\in\mathbb R^d$ denote its
feature vector and let $c_i\in\{1,\ldots,K\}$ denote a latent mixture
component, with $q(c_i=k)=\pi_k$. Conditioned on $c_i$, PEMix models node
features and structural position independently with diagonal Gaussian
distributions,
$q_\eta(\mathbf x_i\mid c_i=k)
=
\mathcal N(
\mathbf x_i\mid
\boldsymbol\mu_k^X,
\operatorname{diag}((\boldsymbol\sigma_k^X)^2))$
and
$q_\eta(\mathbf p_i\mid c_i=k)
=
\mathcal N(
\mathbf p_i\mid
\boldsymbol\mu_k^P,
\operatorname{diag}((\boldsymbol\sigma_k^P)^2))$.

Let
$\Omega_i=\{j\in\{1,\ldots,d\}:M_{ij}=0\}$
denote the observed coordinates of node $i$, and let
$\bar\Omega_i=\{1,\ldots,d\}\setminus\Omega_i$ denote the missing ones.
Conditioning jointly on the observed coordinates and on structural position
gives
\begin{equation}
\label{eq:pemix-responsibility}
\gamma_{ik}
\propto
\pi_k\,
q_\eta(\mathbf p_i\mid c_i=k)
\prod_{j\in\Omega_i}
\mathcal N
\left(
X_{ij}\mid\mu^X_{kj},(\sigma^X_{kj})^2
\right),
\end{equation}
with normalization over $k$. Equivalently, the structural term induces a
node-dependent prior
$\pi_k(\mathbf p_i)=q_\eta(c_i=k\mid\mathbf p_i)$, which is then updated using
the observed feature coordinates. The distinction from a feature-only mixture is therefore in the component
responsibilities. Without structural conditioning, every node starts from the
same prior $\pi_k$; in PEMix, structurally different nodes can have different
component probabilities before their observed features are considered. This is
particularly useful when only limited feature evidence remains. Under the diagonal covariance assumption, the conditional distribution of the
missing coordinates is
\begin{equation}
\label{eq:pemix-conditional}
q_\eta(
\mathbf x_{i,\bar\Omega_i}
\mid
\mathbf x_{i,\Omega_i},\mathbf p_i)
=
\sum_{k=1}^{K}
\gamma_{ik}
\prod_{j\in\bar\Omega_i}
\mathcal N
\left(
X_{ij}\mid\mu^X_{kj},(\sigma^X_{kj})^2
\right).
\end{equation}
Thus, graph structure affects the distribution used to impute missing
coordinates before message passing, rather than only determining how observed
values are propagated between nodes.

\paragraph{End-to-End Prediction}

PEMix retains the conditional uncertainty in Equation
\ref{eq:pemix-conditional} instead of committing to a single point estimate.
Following \citet{TAGUCHI2021155}, the corresponding component-wise
moments are marginalized analytically through the first message-passing layer,
after which the remaining GNN layers are unchanged. The positional encodings
affect only the mixture responsibilities in
\eqref{eq:pemix-responsibility}; they are not concatenated to the node
representation passed to the downstream GNN. 
Following~\citep{TAGUCHI2021155}, who showed that jointly optimizing the mixture and downstream GNN can improve missing-feature reconstruction, we initialize the mixture parameters with EM and subsequently optimize them jointly with the GNN using the downstream task loss. We validate this design choice in \Cref{app:gmm_joint}, where joint optimization consistently improves downstream performance compared with keeping the EM-initialized mixture fixed.
Implementation details, including initialization,
covariance parameterization, and positional-encoding dimension, are provided
in \Cref{app:implementation}. Structural conditioning does not change the
asymptotic cost of the probabilistic first layer: PEMix has the same per-epoch
complexity as GCNmf, up to a dominated $O(nKq)$ term and a one-off partial
eigendecomposition per graph (\Cref{app:complexity}).

\section{Experimental Evaluation}
\label{sec:experiments}

We address three research questions:
\begin{itemize}[leftmargin=*]
    \item \textbf{RQ1:} Does PEMix improve robustness under challenging graph-structured missingness while remaining competitive under standard random missingness?
    \item \textbf{RQ2:} How does PEMix perform under naturally occurring missingness?
    \item \textbf{RQ3:} Are the gains of PEMix due simply to access to positional encodings, or specifically to using them to condition the missing-feature distribution?
\end{itemize}

\paragraph{Datasets and protocol.}
To answer RQ1, we consider four datasets with artificially injected missingness, namely TADPOLE \citep{zhu2019multi}, GraphMSO, PEMS-BAY and Engrad \citep{marisca2024graph}. Since these datasets are complete, the missingness
mechanism and global missingness rate $\mu$ can be controlled directly. In particular we injected RT and U-MCAR with $\mu\in\{0,0.1,\ldots,0.9,0.99\}$. For RQ2, we instead consider datasets where missingness is part of the original
data. AQI, PV-US, and METR-LA are the air-quality, photovoltaic, and traffic datasets used by \citet{marisca2024graph}. We additionally evaluate
PEMix in a different setting using RelBench \citep{relbench}, a benchmark of
relational databases with naturally incomplete features. For forecasting datasets, we convert each prediction time into an independent static graph sample using only past observations; the temporal windows, graph construction, and leakage-prevention protocol are detailed in Appendix~\ref{app:datasets}. The dataset overview is provided in Appendix~\ref{app:datasets}, Table~\ref{tab:datasets}. On the standard graph datasets, we compare against simple baseline with zero imputation (GNN-Zero), median imputation (GNN-Median) and mean imputation (GNN-MI), probabilistic and mask-aware approaches
(GCNmf and GNN-MIM), and propagation methods (FP, PCFI, and FISF).
RelBench requires a relational architecture and has substantially fewer
missing-feature baselines. We therefore compare against the standard Relational
Deep Learning model (RDL)~\citep{relbench} and its missing-indicator extension RDL-MIM~\citep{ferrini2026rethinking}, and
adapt PEMix to the same backbone, denoted RDL-PEMix. The datasets and evaluation
regimes are summarized in \Cref{tab:datasets}.

\textbf{RQ1: PEMix gains most where graph-structured missingness is challenging, without sacrificing competitiveness under U-MCAR.
} For each controlled dataset, at each missingness rate, we evaluate all methods over five
random seeds. The complete curves, including means and standard deviations at
each rate, are reported in \Cref{app:rt-results}. For compactness, in
\Cref{tab:rq1_injected} we summarize each curve by its trapezoidal area under
the curve (AUC) over the evaluated missingness rates. Higher AUC is better for
macro-F1, while lower AUC is better for MAE. Under RT, the setting PEMix is explicitly designed to address, PEMix obtains the best AUC on three of the four datasets: TADPOLE, GraphMSO, and PEMS-BAY, while ranking third on Engrad. The gains are particularly pronounced on TADPOLE and PEMS-BAY, where PEMix improves substantially over the strongest competing baselines. This supports the intended role of structural conditioning when missing observations form graph-localized regions and local propagation becomes less informative. Under U-MCAR, in contrast, we would not necessarily expect PEMix to dominate, since missing entries are scattered independently and propagation-based or simpler recovery mechanisms are not exposed to the same structural bottleneck. Nevertheless, PEMix remains competitive: it ranks first on TADPOLE and GraphMSO, second on PEMS-BAY, and third on EngRAD. Thus, the benefit of PEMix is concentrated in the challenging graph-structured regime for which it was designed, while its use does not generally entail a loss of competitiveness when missingness is unstructured.
\begin{table*}[t]
\centering
\caption{Performance under controlled RT and U-MCAR, reported as trapezoidal AUC over $\mu$. Blue indicates the best model in each column, orange the second one, and green the third one. Full tables in Appendix~\ref{app:rt-results}.}
\label{tab:rq1_injected}
\scriptsize
\setlength{\tabcolsep}{3pt}

\begin{tabular}{lcccccccc}
\toprule
& \multicolumn{4}{c}{RT} & \multicolumn{4}{c}{U-MCAR} \\
\cmidrule(lr){2-5}\cmidrule(lr){6-9}

Model 
& TADPOLE & GraphMSO & PEMS-BAY & Engrad
& TADPOLE & GraphMSO & PEMS-BAY & Engrad \\

& Macro-F1 $\uparrow$ 
& MAE $\downarrow$ 
& MAE $\downarrow$ 
& MAE $\downarrow$
& Macro-F1 $\uparrow$ 
& MAE $\downarrow$ 
& MAE $\downarrow$ 
& MAE $\downarrow$ \\

\midrule

GNN-Zero 
& \color{rankthree}0.7388 
& \color{rankthree}1.369627 
& 4.294239 
& 7.345723 
& 0.7565 
& \color{rankthree}1.368397 
& 4.059947 
& 5.796827 \\

GNN-Median 
& 0.7015 
& 1.369711 
& 4.340867 
& 7.317703 
& 0.7398 
& 1.368436 
& 4.038923 
& 5.859984 \\

GNN-MI 
& 0.7130 
& 1.369657 
& 4.251831 
& 7.322302 
& 0.7421 
& 1.368595 
& 4.079495 
& 5.816402 \\

FP 
& 0.7254 
& 1.371720 
& 4.203665 
& \color{ranktwo}6.515184 
& 0.7473 
& 1.370337 
& 4.035542 
& \color{rankone}\textbf{5.504311} \\

PCFI 
& 0.7272 
& 1.372185 
& 4.308706 
& 7.982851 
& \color{ranktwo}0.7593 
& 1.370070 
& 4.115370 
& 6.375060 \\

FISF 
& 0.7232 
& 1.373884 
& 4.237767 
& \color{rankone}\textbf{6.468120} 
& 0.7566 
& 1.372575 
& 4.071262 
& \color{ranktwo}5.603992 \\

GCNmf 
& 0.7148 
& \color{ranktwo}1.369023 
& \color{ranktwo}3.921914 
& 7.373967 
& \color{rankthree}0.7588 
& \color{ranktwo}1.367092 
& \color{rankone}\textbf{3.393461} 
& 6.060386 \\

GNN-MIM 
& \color{ranktwo}0.7469 
& 1.371282 
& \color{rankthree}4.193381 
& 7.584628 
& 0.7482 
& 1.369593 
& \color{rankthree}3.909064 
& 5.838703 \\

PEMix 
& \color{rankone}\textbf{0.8243} 
& \color{rankone}\textbf{1.368784} 
& \color{rankone}\textbf{3.764460} 
& \color{rankthree}7.206951 
& \color{rankone}\textbf{0.8094} 
& \color{rankone}\textbf{1.366935} 
& \color{ranktwo}3.463835 
& \color{rankthree}5.753445 \\

\bottomrule
\end{tabular}
\end{table*}

\textbf{RQ2: PEMix performs strongly under naturally occurring missingness.}

\begin{wraptable}{r}{0.58\textwidth}
\centering
\caption{Performance with naturally occurring missingness (test MAE, mean $\pm$ std over five seeds; lower is better). Blue, orange, and green mark the three best models within each dataset.}
\label{tab:native_forecasting}
\scriptsize
\setlength{\tabcolsep}{4pt}
\resizebox{\linewidth}{!}{%
\begin{tabular}{lccc}
\toprule
Model & AQI & PV-US & METR-LA \\
\midrule
GNN-Zero & $20.1895\pm0.2203$ & $12.4871\pm0.0151$ & $6.5199\pm0.0465$ \\
GNN-Median & $19.7317\pm0.2595$ & $12.3410\pm0.0289$ & $6.6117\pm0.0760$ \\
GNN-MI & $19.9730\pm0.1712$ & $12.4844\pm0.0337$ & \color{rankthree}$6.5197\pm0.0576$ \\
FP & \color{rankthree}$18.9450\pm0.0409$ & $12.5251\pm0.0290$ & $6.5945\pm0.0487$ \\
PCFI & $19.5384\pm0.0707$ & $12.9281\pm0.0698$ & $6.6484\pm0.0957$ \\
FISF & $19.0586\pm0.0778$ & $12.5746\pm0.0203$ & $6.5991\pm0.0721$ \\
GCNmf & \color{ranktwo}$18.9071\pm0.0460$ & \color{rankone}$\mathbf{11.9906\pm0.1444}$ & \color{ranktwo}$6.5170\pm0.0182$ \\
GNN-MIM & $19.4978\pm0.1793$ & \color{rankthree}$12.2302\pm0.0135$ & $6.7227\pm0.0778$ \\
PEMix & \color{rankone}$\mathbf{18.7274\pm0.0523}$ & \color{ranktwo}$12.0225\pm0.0392$ & \color{rankone}$\mathbf{6.4511\pm0.0105}$ \\
\bottomrule
\end{tabular}%
}
\end{wraptable}
Results on datasets with naturally occurring missingness are reported in Table~\ref{tab:native_forecasting}. PEMix achieves the lowest MAE on two of the three benchmarks, AQI and METR-LA, while ranking second on PV-US. This pattern is consistent with the spatial organization of missingness in these datasets. We quantify this structure through the excess co-missingness of neighboring nodes: intuitively, we compare how often the two endpoints of an edge are simultaneously missing with how often this would occur under spatially independent missingness. Values above one therefore indicate that missing observations cluster over the graph. The exact definition and dataset-level values are reported in Appendix~\ref{app:spatial-natural}. All three datasets exhibit co-missingness above the range expected under spatially independent missingness. The effect is stronger on METR-LA, whose ratio reaches 4.03 over timestamos with partial missingness, compared with 1.16 on AQI and 1.21 on PV-US. PEMix achieves the best MAE on AQI and METR-LA, whereas on PV-US GCNmf performs best and PEMix remains a close second.
The same trend extends beyond standard graph benchmarks: in the relational setting, RDL-PEMix obtains the best result on four of the six RelBench tasks reported in Table~\ref{tab:relbench_results}, suggesting that PEMix can be integrated into relational models as an effective strategy for handling missing features.

\begin{table*}[t]
\centering
\caption{RelBench results with naturally occurring missingness (mean $\pm$
std over three seeds; ten epochs). RDL-PEMix denotes the
relational adaptation of PEMix, while RDL-MIM is the missing-indicator
adaptation of RDL. Bold marks the best mean
within each task.}
\label{tab:relbench_results}
\small\setlength{\tabcolsep}{3pt}
\scriptsize
\renewcommand{\arraystretch}{1.0}
\begin{tabular}{llcccc}
\toprule
Dataset & Task & Metric & RDL & RDL-MIM & RDL-PEMix \\
\midrule
\multicolumn{6}{l}{\textbf{Node regression}} \\
\midrule
\textsc{rel-trial} & study-adverse & MAE$\downarrow$ & $44.470\pm0.210$ & $43.115\pm0.155$ & $\mathbf{42.840\pm0.680}$ \\
\textsc{rel-trial} & site-success & MAE$\downarrow$ & $0.400\pm0.020$ & $\mathbf{0.375\pm0.017}$ & $0.391\pm0.008$ \\
\textsc{rel-f1} & driver-position & MAE$\downarrow$ & $\mathbf{4.022\pm0.120}$ & $4.0271\pm 0.0389$ & $4.208\pm0.090$ \\
\midrule
\multicolumn{6}{l}{\textbf{Node classification}} \\
\midrule
\textsc{rel-trial} & study-outcome & AUC$\uparrow$ & $0.686\pm0.010$ & $0.6794\pm0.0051$ & $\mathbf{0.706\pm0.010}$ \\
\textsc{rel-f1} & driver-top3 & AUC$\uparrow$ & $0.755\pm0.006$ & $0.768\pm0.004$ & $\mathbf{0.773\pm0.010}$ \\
\textsc{rel-stack} & user-engagement & AUC$\uparrow$ & $0.9062\pm0.0002$ & $0.9060\pm0.0008$ & $\mathbf{0.9072\pm0.0001}$ \\
\bottomrule
\end{tabular}
\end{table*}

\textbf{RQ3: Structural conditioning, rather than access to PE alone, drives PEMix’s gains.}

\begin{wraptable}{r}{0.58\textwidth}
\centering
\caption{Effect of PE on TADPOLE and PV-US.
"Mechanism" denotes methods in which PE are used while handling
missing features, whereas "Downstream" denotes variants where PE
are concatenated only after missing-feature handling.
We report macro-F1 AUC across missingness rates for TADPOLE
(higher is better) and test MAE under natural missingness for PV-US
(lower is better). PV-US results are means $\pm$ std over five seeds.}
\label{tab:rq3_pe}
\scriptsize
\setlength{\tabcolsep}{4pt}
\resizebox{\linewidth}{!}{%
\begin{tabular}{llcc}
\toprule
PE & Model
& TADPOLE & PV-US \\
usage & & Macro-F1 AUC $\uparrow$ & MAE $\downarrow$ \\
\midrule
\multirow{4}{*}{Mechanism}
& PEMix
& \textbf{0.8243}
& $\mathbf{12.0225 \pm 0.0392}$ \\
& FP + PE
& 0.7336
& $12.4642 \pm 0.0221$ \\
& PCFI + PE
& 0.7474
& $12.8643 \pm 0.0454$ \\
& FISF + PE
& 0.7573
& $12.5981 \pm 0.0443$ \\
\midrule
\multirow{4}{*}{Downstream}
& GCNmf + PE
& 0.7211
& $12.1421 \pm 0.0145$ \\
& FP + PE
& 0.7754
& $12.1261 \pm 0.0452$ \\
& PCFI + PE
& 0.7745
& $12.6979 \pm 0.0790$ \\
& FISF + PE
& 0.7798
& $12.2364 \pm 0.0649$ \\
\midrule
-- & PE only
& 0.2923
& $16.5526 \pm 0.0401$ \\
\bottomrule
\end{tabular}%
}
\end{wraptable}
To isolate the role of positional encodings, we compare two alternatives.
First, we provide the same PE to propagation-based methods inside their missing-feature mechanism. Second, we concatenate PE only after missing-feature
handling, before the downstream predictor. We also include a PE-only control.
We conducted this ablation on the Tadpole and PV-US dataset. As in RQ1, for the Tadpole dataset we summarize each missingness curve by its macro-F1 AUC; complete results with average and standard deviations over five seeds for each level of missingness $\mu$ are reported in Appendix~\ref{app:ablations}. For PV-US we report means and standard deviations over five seeds. Results are reported in Table~\ref{tab:rq3_pe}. Simply providing positional encodings is not sufficient to reproduce PEMix’s performance. Adding the same PE to propagation-based missing-feature mechanisms leaves a substantial gap, while using PE only downstream improves performance on TADPOLE but remains below PEMix. The PE-only control performs poorly on both datasets. These results suggest that PEMix benefits specifically from using structural information to condition the missing-feature distribution, rather than from access to positional encodings per se.

\section{Conclusion}
We studied missing node features when missingness itself is structured by the graph. We formalized graph-structured mask dependence and introduced RT, a controlled mechanism that exposes a boundary bottleneck of propagation-based method. Motivated by this limitation, we proposed PEMix, which conditions a probabilistic model of incomplete node features on global structural representations. Across controlled and naturally occurring missingness, PEMix is particularly effective when missing observations are graph-localized, while remaining competitive under standard missingness. Ablations further show that the gains arise from using structural information to condition the missing-feature distribution, rather than from access to positional encodings alone. These results highlight graph structure as a useful source of information not only for prediction, but also for modeling incomplete node features. \textbf{Limitations:} PEMix assumes diagonal mixture covariances and a fully observed graph, and our evaluation is limited to node classification and forecasting. Broader analyses of structural encodings, temporal modeling, and other downstream tasks are left for future work.

\section*{Acknowledgments}
The work is partially funded by the Norwegian Research Council through \emph{RELAY: Relational Deep Learning for Energy Analytics} (project 345017), and the Centre of Excellence Integreat - The Norwegian Centre for knowledge-driven machine learning (project 332645). The authors gratefully thank NVIDIA Corporation for donating the GPUs used in this project.

\begingroup
\raggedright
\bibliographystyle{iclr2027_conference}
\bibliography{references}
\endgroup

\clearpage
\appendix
\section{Proofs}
\label{app:proofs}
\subsection{RT is graph-structured MCAR}
\noindent\textbf{Proposition~\ref{prop:rt_graph_structured_mcar} (RT is graph-structured MCAR), restated.}
If source selection and tie-breaking are independent of $\mathbf{X}$ and $\mathbf{Y}$ conditional on $\mathbf{A}$,
then RT is MCAR in the sense of~\eqref{eq:mcar}. Under the non-degeneracy and graph-support conditions stated in Appendix~\ref{app:proofs}, it also satisfies
Definition~\ref{def:graph_structured_dependence}.

\begin{proof}
We work conditional on the realized adjacency matrix $\mathbf A$ and fixed feature budgets, as in the main text. For the graph-structured conclusion, we make the necessary non-degeneracy conditions explicit: the graph is connected and undirected, the source distribution has full support on $V$, and some channel has $2\le m_j\le n-1$. For the topology comparison below, the complete graph and the star on $V$ are assumed to have positive probability under the distribution of $\mathbf A$. In comparing graphs, the source law and all budgets are held fixed. These conditions are needed for that conclusion: at zero or full missingness the mask is deterministic, and if all budgets are at most one, no growth step need consult an edge. Thus independence from the data alone proves MCAR, while the graph-structured assertion is understood in this non-degenerate regime.

\emph{MCAR.} Let $U$ collect the source and all tie-breaking variables. The recursion is a deterministic map $\mathbf M=F(\mathbf A,U)$: each update uses only adjacency, the current region, and $U$. Since $U$ is independent of $(\mathbf X,\mathbf Y)$ conditional on $\mathbf A$, for every mask $m$,
\[
\Pr(\mathbf M=m\mid\mathbf X,\mathbf Y,\mathbf A)
=\Pr(F(\mathbf A,U)=m\mid\mathbf A)
=\Pr(\mathbf M=m\mid\mathbf A).
\]
This is Equation~\ref{eq:mcar}; dependence within the mask does not affect this conditional independence from the data.

\emph{Internal dependence.} Fix a channel with $2\le m_j\le n-1$. Every node has positive probability of being masked, since it can be the source. The fixed budget implies $\sum_iM_{ij}=m_j<n$ almost surely, so at least one indicator also has positive probability of being zero. If the channel indicators were mutually independent, then
\[
0=\operatorname{Var}\!\left(\sum_iM_{ij}\mid\mathbf A\right)
=\sum_i\operatorname{Var}(M_{ij}\mid\mathbf A)>0,
\]
which is impossible. Independence of every entry of the full mask would imply independence within this channel. The conditional distribution $p(\mathbf M\mid\mathbf A)$ therefore does not factorize.

\emph{Dependence on topology.} Compare the complete graph with adjacency matrix $\mathbf A$ on the labeled set $V$ with a star centered at $a$ with adjacency matrix $\mathbf A'$, holding the full-support source law, budgets, and uniform tie-breaking rule fixed. Choose $S\subseteq V\setminus\{a\}$ with $|S|=m_j$. In the complete graph all boundary candidates have the same score. Starting in $S$ and then selecting its remaining nodes has positive probability, so $\Pr(C_j=S\mid\mathbf A)>0$. On the star every RT region is connected, because it starts with one node and grows by adding boundary nodes. A set of at least two leaves excluding the center is disconnected, so $\Pr(C_j=S\mid\mathbf A')=0$. The conditional channel marginals, and hence the joint mask distributions $p(\mathbf M\mid\mathbf A)$ and $p(\mathbf M\mid\mathbf A')$, differ. Together with internal dependence, this proves Definition~\ref{def:graph_structured_dependence} under the stated non-degeneracy conditions.
\end{proof}

\subsection{Boundary bottleneck for feature propagation}
\noindent\textbf{Proposition~\ref{prop:boundary_bottleneck} (Boundary bottleneck for feature propagation), restated.}
Let $\rho(\widehat{\mathbf{A}}_{MM})$ denote the spectral radius of
$\widehat{\mathbf{A}}_{MM}$ and assume
$\rho(\widehat{\mathbf{A}}_{MM})<1$. Then the iteration in
~\eqref{eq:fp_block} converges to
\begin{equation*}
    \widehat{x}_M
    =
    \mathbf{B}x_O,
    \qquad
    \mathbf{B}
    =
    \left(
        \mathbf{I}-\widehat{\mathbf{A}}_{MM}
    \right)^{-1}
    \widehat{\mathbf{A}}_{MO}.
\end{equation*}
Moreover,
\begin{equation*}
    \operatorname{rank}(\mathbf{B})
    \leq
    |\partial_O M|.
\end{equation*}

\begin{proof}
Conditional on the realized graph, the normalized adjacency is a fixed local propagation matrix: an off-diagonal entry can be nonzero only on a graph edge. Observations $x_O$ remain clamped. Set $T=\widehat{\mathbf A}_{MM}$ and $C=\widehat{\mathbf A}_{MO}$. For every integer $t\ge1$, induction on the update gives
\[
x_M^{(t)}=T^t x_M^{(0)}+\sum_{r=0}^{t-1}T^r Cx_O.
\]
Since $\rho(T)<1$, every Jordan block of $T^t$ tends to zero: its entries are finite sums of polynomial factors in $t$ multiplied by powers of eigenvalues of modulus less than one (nilpotent blocks vanish after finitely many powers). Consequently $T^t\to0$. Also $I-T$ is invertible, and the finite geometric-series identity gives
\[
\sum_{r=0}^{t-1}T^r=(I-T)^{-1}(I-T^t)\longrightarrow(I-T)^{-1}.
\]
For any finite initial vector, the initial-condition term vanishes and the limit is $(I-T)^{-1}Cx_O=\mathbf Bx_O$. Invertibility also makes this the unique fixed point.

If $u\in O\setminus\partial_O M$, there is no edge from a missing node to $u$, so locality gives $C_{:,u}=0$. Left multiplication by $(I-T)^{-1}$ preserves that zero column. Thus only observed boundary coordinates can influence the limit. Since $C$ has at most $|\partial_O M|$ nonzero columns,
\[
\operatorname{rank}(\mathbf B)=\operatorname{rank}(C)\le |\partial_O M|,
\]
where the equality follows from invertibility of $I-T$.
\end{proof}
The result concerns fixed linear, channel-wise propagation. It does not automatically extend to nonlinear or cross-channel recovery, and a low-rank recovery map alone does not imply positive error for every graph signal.


\section{Datasets and Evaluation Protocol}
\label{app:datasets}

\paragraph{Static formulation of forecasting datasets.} For AQI, PV-US, METR-LA, and GraphMSO, we convert each prediction time into an independent static graph sample. The graph topology is fixed over time and constructed from sensor locations, official distance files, or the GraphMSO generator, depending on the dataset; it never uses future observations. Node features contain only past sensor measurements, flattened over the input window, together with unmasked calendar covariates when available. Targets are the subsequent measurements. The input/output windows are 24/6 for AQI, 72/6 for PV-US, 24/12 for METR-LA, and 72/36 for GraphMSO. No temporal model is used: all methods operate on the resulting static graph samples with the same inputs.
To prevent leakage, train/validation/test splits are chronological and windows are constructed so that each input–target pair lies entirely within one split. Standardization and imputation statistics are estimated from observed training values only. Artificial masks are generated on the raw series before windowing, so repeated occurrences of the same measurement share the same missingness state. Naturally missing values are preserved, and missing targets are excluded from both loss and evaluation. PEMix positional encodings depend only on the fixed graph and are computed once per dataset; propagation-based imputers operate only on sensor measurements, not on calendar covariates.

\begin{table}[H]
\centering
\caption{Datasets used in the controlled and native-missingness experiments.
Controlled datasets are originally complete and missingness is injected
artificially. Native datasets contain missing values in the original data.
For native datasets, we report the percentage of missing entries in the
original data.}
\label{tab:datasets}
\small
\setlength{\tabcolsep}{5pt}
\begin{tabular}{llllc}
\toprule
Regime & Dataset & Domain & Source & Missing (\%) \\
\midrule
\multirow{5}{*}{Controlled}
& TADPOLE & clinical & \citet{zhu2019multi} & -- \\
& GraphMSO & graph forecasting & \citet{marisca2024graph} & -- \\
& PEMS-BAY & traffic forecasting & \citet{marisca2024graph} & -- \\
& Engrad & forecasting & \citet{marisca2024graph} & -- \\
\midrule
\multirow{6}{*}{Native}
& AQI & air-quality forecasting & \citet{marisca2024graph} & 25.67 \\
& PV-US & photovoltaic forecasting & \citet{marisca2024graph} & 54.02 \\
& METR-LA & traffic forecasting & \citet{marisca2024graph} & 8.11 \\
& RelBench (rel-f1) & relational databases & \citet{relbench} & 9.42 \\
& RelBench (rel-trial) & relational databases & \citet{relbench} & 23.90 \\
& RelBench (rel-stack) & relational databases & \citet{relbench} & 19.13 \\
\bottomrule
\end{tabular}
\end{table}

The controlled experiments comprise one classification dataset, namely TADPOLE, and
three forecasting datasets (GraphMSO, PEMS-BAY, and Engrad). Native missingness is evaluated on
AQI, PV-US, METR-LA, and the six RelBench tasks in Table~\ref{tab:relbench_results}.

Controlled curves are integrated with the trapezoidal rule over $\mu\in\{0,0.1,\dots,0.9,0.99\}$,
without division by 0.99. Classification results use five seeds (1, 43, 15, 118, 222), with the
same splits for every model and missingness mechanism; forecasting uses five test seeds per point.
Native forecasting uses five seeds, and RelBench uses three seeds and ten epochs. AUC differences
are reported descriptively; we do not perform significance tests.

\subsection{Spatial organization of missingness in real-world datasets}
\label{app:spatial-natural}

We test whether naturally occurring missing values fall on neighboring nodes more often than they
would if missingness were spatially independent.

\paragraph{Co-missingness ratio.}
Let $M_{tij}=1$ if feature $j$ of node $i$ is missing at timestamp $t$. For each feature $j$ and
timestamp $t$, let
\[
p_{tj}=\frac{1}{n}\sum_{i=1}^{n} M_{tij},
\qquad
h^{\mathrm{miss}}_{tj}=\frac{1}{|E|}\sum_{(u,v)\in E}\mathbb{I}\!\left[M_{tuj}=1 \land M_{tvj}=1\right]
\]
denote the missingness rate and the fraction of edges whose two endpoints are both missing.
If missingness were independent across nodes with the same rate $p_{tj}$, the expected fraction
of missing--missing edges would be $p_{tj}^{2}$. Over a set of timestamps $\mathcal{T}$, we
therefore define
\[
HC_{\mathrm{miss}}(j)=\frac{\sum_{t\in\mathcal{T}} h^{\mathrm{miss}}_{tj}}{\sum_{t\in\mathcal{T}} p_{tj}^{2}},
\]
a ratio of sums rather than an average of per-timestamp ratios. $HC_{\mathrm{miss}}=1$
corresponds to spatial independence, and values above $1$ indicate excess co-missingness between
neighboring nodes. Each dataset has a single measured channel (PM2.5 for AQI, power for PV-US,
and speed for METR-LA), so the dataset-level value is that of its channel.

\paragraph{Evaluation subsets.}
We compute the ratio over all timestamps and over \emph{partial} timestamps, those with
$0<p_{tj}<1$. Completely observed and completely missing snapshots carry no spatial contrast.
This distinction matters on METR-LA, where many missing values come from graph-wide outages.

\paragraph{Permutation control.}
To check that excess co-missingness reflects graph position rather than marginal missingness
rates, we randomly permute the complete missingness trajectories of nodes across graph positions.
This preserves temporal missingness patterns and the number of missing nodes at each timestamp,
while breaking their association with the topology. We use 999 permutations over all timestamps
and report the central $95\%$ range of the resulting null distribution.

\paragraph{Results.}
Table~\ref{tab:native_missingness_homophily} reports the results. On all three datasets, the
observed ratio lies above its permutation null range, so co-missingness is associated with graph
position. Over all timestamps the excess is largest on AQI; over partial timestamps it is largest
on METR-LA ($4.03$), followed by PV-US ($1.21$) and AQI ($1.16$).

\begin{table}[t]
\centering
\caption{Co-missingness ratio $HC_{\mathrm{miss}}$ on datasets with naturally occurring
missingness. Values above $1$ indicate that neighboring nodes are missing together more often
than under spatially independent missingness. The permutation null is the central $95\%$ range
over 999 node permutations, computed over all timestamps.}
\label{tab:native_missingness_homophily}
\small
\begin{tabular}{lccc}
\toprule
Dataset & All timestamps & Partial timestamps & Permutation null ($95\%$) \\
\midrule
AQI     & 1.15 & 1.16 & $[0.902,\,1.118]$ \\
PV-US   & 1.03 & 1.21 & $[0.999,\,1.001]$ \\
METR-LA & 1.04 & 4.03 & $[0.997,\,1.003]$ \\
\bottomrule
\end{tabular}
\end{table}

\section{Propagation Controls on TADPOLE}
\label{app:propagation-controls}
Tables~\ref{tab:tadpole-fp}--\ref{tab:tadpole-fisf} give the complete macro-F1 and classification AUROC results for Figure~\ref{fig:rt_umcar_smcar} (mean $\pm$ standard deviation over five seeds).

For a metric $s(\mu)$, we integrate its mean curve as $\sum_k(\mu_{k+1}-\mu_k)[s(\mu_k)+s(\mu_{k+1})]/2$ over $[0,0.99]$, without normalisation. Integrating the rounded table means gives average macro-F1 areas of $0.7252$, $0.7543$, and $0.7508$ under RT, U-MCAR, and S-MCAR, respectively. The corresponding areas under the classification-AUROC curves are $0.8855$, $0.8958$, and $0.8963$. These curve areas are distinct from the classification AUROC at a particular missingness rate.
\begin{table}[p]
\centering\footnotesize\setlength{\tabcolsep}{3pt}
\caption{Complete FP results on TADPOLE (five seeds). Higher macro-F1 and AUROC are better.}
\label{tab:tadpole-fp}
\begin{tabular}{rcccccc}
\toprule
& \multicolumn{3}{c}{Macro-F1} & \multicolumn{3}{c}{AUROC} \\
\cmidrule(lr){2-4}\cmidrule(lr){5-7}
$\mu$ (\%) & RT & U-MCAR & S-MCAR & RT & U-MCAR & S-MCAR \\
\midrule
0 & $0.806\pm0.065$ & $0.806\pm0.065$ & $0.806\pm0.065$ & $0.929\pm0.022$ & $0.929\pm0.022$ & $0.929\pm0.022$ \\
10 & $0.805\pm0.022$ & $0.836\pm0.038$ & $0.842\pm0.028$ & $0.932\pm0.013$ & $0.939\pm0.016$ & $0.939\pm0.019$ \\
20 & $0.813\pm0.035$ & $0.808\pm0.044$ & $0.797\pm0.038$ & $0.941\pm0.010$ & $0.934\pm0.021$ & $0.932\pm0.014$ \\
30 & $0.797\pm0.034$ & $0.798\pm0.037$ & $0.806\pm0.065$ & $0.935\pm0.015$ & $0.928\pm0.016$ & $0.924\pm0.031$ \\
40 & $0.802\pm0.054$ & $0.784\pm0.026$ & $0.726\pm0.086$ & $0.926\pm0.019$ & $0.924\pm0.013$ & $0.897\pm0.022$ \\
50 & $0.766\pm0.060$ & $0.782\pm0.051$ & $0.720\pm0.116$ & $0.917\pm0.019$ & $0.925\pm0.027$ & $0.906\pm0.026$ \\
60 & $0.774\pm0.075$ & $0.788\pm0.044$ & $0.781\pm0.040$ & $0.931\pm0.018$ & $0.921\pm0.031$ & $0.918\pm0.018$ \\
70 & $0.729\pm0.144$ & $0.735\pm0.048$ & $0.748\pm0.064$ & $0.893\pm0.091$ & $0.897\pm0.010$ & $0.908\pm0.014$ \\
80 & $0.676\pm0.120$ & $0.770\pm0.052$ & $0.721\pm0.066$ & $0.862\pm0.051$ & $0.915\pm0.028$ & $0.890\pm0.028$ \\
90 & $0.538\pm0.078$ & $0.695\pm0.060$ & $0.609\pm0.118$ & $0.811\pm0.068$ & $0.904\pm0.030$ & $0.872\pm0.049$ \\
99 & $0.392\pm0.074$ & $0.241\pm0.000$ & $0.316\pm0.103$ & $0.673\pm0.044$ & $0.492\pm0.027$ & $0.604\pm0.119$ \\
\bottomrule\end{tabular}\end{table}
\begin{table}[p]
\centering\footnotesize\setlength{\tabcolsep}{3pt}
\caption{Complete PCFI results on TADPOLE (five seeds). Higher macro-F1 and AUROC are better.}
\label{tab:tadpole-pcfi}
\begin{tabular}{rcccccc}
\toprule
& \multicolumn{3}{c}{Macro-F1} & \multicolumn{3}{c}{AUROC} \\
\cmidrule(lr){2-4}\cmidrule(lr){5-7}
$\mu$ (\%) & RT & U-MCAR & S-MCAR & RT & U-MCAR & S-MCAR \\
\midrule
0 & $0.796\pm0.054$ & $0.796\pm0.054$ & $0.796\pm0.054$ & $0.927\pm0.016$ & $0.927\pm0.016$ & $0.927\pm0.016$ \\
10 & $0.795\pm0.020$ & $0.810\pm0.046$ & $0.818\pm0.038$ & $0.928\pm0.008$ & $0.936\pm0.023$ & $0.938\pm0.015$ \\
20 & $0.825\pm0.032$ & $0.811\pm0.038$ & $0.833\pm0.021$ & $0.940\pm0.018$ & $0.938\pm0.020$ & $0.940\pm0.013$ \\
30 & $0.810\pm0.019$ & $0.813\pm0.043$ & $0.797\pm0.036$ & $0.938\pm0.017$ & $0.934\pm0.021$ & $0.928\pm0.022$ \\
40 & $0.805\pm0.047$ & $0.804\pm0.028$ & $0.822\pm0.047$ & $0.929\pm0.016$ & $0.933\pm0.011$ & $0.929\pm0.016$ \\
50 & $0.789\pm0.059$ & $0.800\pm0.035$ & $0.799\pm0.023$ & $0.922\pm0.021$ & $0.927\pm0.020$ & $0.931\pm0.017$ \\
60 & $0.778\pm0.078$ & $0.785\pm0.049$ & $0.798\pm0.016$ & $0.924\pm0.025$ & $0.922\pm0.024$ & $0.923\pm0.009$ \\
70 & $0.766\pm0.093$ & $0.774\pm0.036$ & $0.795\pm0.039$ & $0.910\pm0.046$ & $0.917\pm0.011$ & $0.921\pm0.013$ \\
80 & $0.601\pm0.166$ & $0.791\pm0.050$ & $0.727\pm0.117$ & $0.828\pm0.079$ & $0.926\pm0.024$ & $0.899\pm0.038$ \\
90 & $0.534\pm0.173$ & $0.734\pm0.046$ & $0.708\pm0.113$ & $0.762\pm0.207$ & $0.911\pm0.028$ & $0.912\pm0.018$ \\
99 & $0.435\pm0.067$ & $0.241\pm0.000$ & $0.372\pm0.125$ & $0.725\pm0.015$ & $0.487\pm0.105$ & $0.626\pm0.175$ \\
\bottomrule\end{tabular}\end{table}
\begin{table}[p]
\centering\footnotesize\setlength{\tabcolsep}{3pt}
\caption{Complete FISF results on TADPOLE (five seeds). Higher macro-F1 and AUROC are better.}
\label{tab:tadpole-fisf}
\begin{tabular}{rcccccc}
\toprule
& \multicolumn{3}{c}{Macro-F1} & \multicolumn{3}{c}{AUROC} \\
\cmidrule(lr){2-4}\cmidrule(lr){5-7}
$\mu$ (\%) & RT & U-MCAR & S-MCAR & RT & U-MCAR & S-MCAR \\
\midrule
0 & $0.806\pm0.040$ & $0.806\pm0.040$ & $0.806\pm0.040$ & $0.930\pm0.016$ & $0.930\pm0.016$ & $0.930\pm0.016$ \\
10 & $0.802\pm0.013$ & $0.825\pm0.058$ & $0.833\pm0.021$ & $0.932\pm0.015$ & $0.939\pm0.019$ & $0.940\pm0.014$ \\
20 & $0.809\pm0.041$ & $0.811\pm0.036$ & $0.805\pm0.042$ & $0.940\pm0.009$ & $0.932\pm0.017$ & $0.934\pm0.012$ \\
30 & $0.784\pm0.055$ & $0.795\pm0.045$ & $0.812\pm0.069$ & $0.930\pm0.018$ & $0.930\pm0.016$ & $0.924\pm0.028$ \\
40 & $0.808\pm0.049$ & $0.790\pm0.045$ & $0.828\pm0.032$ & $0.932\pm0.013$ & $0.933\pm0.014$ & $0.928\pm0.010$ \\
50 & $0.789\pm0.040$ & $0.797\pm0.035$ & $0.797\pm0.037$ & $0.924\pm0.011$ & $0.934\pm0.022$ & $0.919\pm0.027$ \\
60 & $0.822\pm0.044$ & $0.781\pm0.060$ & $0.791\pm0.042$ & $0.935\pm0.014$ & $0.915\pm0.023$ & $0.927\pm0.012$ \\
70 & $0.711\pm0.146$ & $0.747\pm0.079$ & $0.783\pm0.045$ & $0.905\pm0.047$ & $0.907\pm0.022$ & $0.920\pm0.023$ \\
80 & $0.614\pm0.145$ & $0.773\pm0.054$ & $0.765\pm0.045$ & $0.834\pm0.088$ & $0.925\pm0.021$ & $0.917\pm0.010$ \\
90 & $0.547\pm0.145$ & $0.774\pm0.061$ & $0.691\pm0.098$ & $0.801\pm0.062$ & $0.917\pm0.027$ & $0.887\pm0.044$ \\
99 & $0.380\pm0.137$ & $0.241\pm0.000$ & $0.356\pm0.158$ & $0.665\pm0.098$ & $0.476\pm0.066$ & $0.607\pm0.213$ \\
\bottomrule\end{tabular}\end{table}

\section{Complete results under controlled missingness}\label{app:rt-results}
The following tables report macro-F1 for TADPOLE and MAE for GraphMSO, PEMS-BAY, and EngRAD. Each dataset is split across two tables (0--50\% and 60--99\%) for legibility. hading marks the three best trapezoidal AUC values within each dataset.
\begin{table*}[p]
\centering
\caption{TADPOLE: complete macro-F1 results under RT (mean $\pm$ standard deviation over five seeds). Missingness $\mu$ is in percent; rates 0--50\%; AUC is reported in the continuation.}
\label{tab:rt-full-tadpole}
\footnotesize\setlength{\tabcolsep}{4pt}
\begin{adjustbox}{max width=\linewidth}
\begin{tabular}{lcccccc}
\toprule
Model & 0 & 10 & 20 & 30 & 40 & 50 \\
\midrule
 GNN-Zero & $0.825\pm0.038$ & $0.811\pm0.026$ & $0.807\pm0.052$ & $0.804\pm0.039$ & $0.815\pm0.034$ & $0.761\pm0.066$ \\
GNN-Median & $0.825\pm0.038$ & $0.797\pm0.030$ & $0.800\pm0.038$ & $0.789\pm0.070$ & $0.802\pm0.045$ & $0.782\pm0.055$ \\
GNN-MI & $0.825\pm0.038$ & $0.805\pm0.033$ & $0.811\pm0.044$ & $0.806\pm0.033$ & $0.812\pm0.051$ & $0.777\pm0.035$ \\
FP & $0.806\pm0.065$ & $0.805\pm0.022$ & $0.813\pm0.035$ & $0.797\pm0.034$ & $0.802\pm0.054$ & $0.766\pm0.060$ \\
PCFI & $0.796\pm0.054$ & $0.795\pm0.020$ & $0.825\pm0.032$ & $0.810\pm0.019$ & $0.805\pm0.047$ & $0.789\pm0.059$ \\
FISF & $0.806\pm0.040$ & $0.802\pm0.013$ & $0.809\pm0.041$ & $0.784\pm0.055$ & $0.808\pm0.049$ & $0.789\pm0.040$ \\
GCNmf & $0.814\pm0.043$ & $0.807\pm0.051$ & $0.818\pm0.071$ & $0.796\pm0.070$ & $0.836\pm0.056$ & $0.770\pm0.154$ \\
 GNN-MIM & $0.828\pm0.038$ & $0.801\pm0.052$ & $0.820\pm0.025$ & $0.797\pm0.054$ & $0.824\pm0.050$ & $0.798\pm0.057$ \\
 PEMix & $0.828\pm0.042$ & $0.864\pm0.041$ & $0.846\pm0.050$ & $0.838\pm0.050$ & $0.844\pm0.054$ & $0.826\pm0.028$ \\
\bottomrule
\end{tabular}
\end{adjustbox}
\end{table*}

\begin{table*}[p]
\centering
\caption{TADPOLE: complete macro-F1 results under RT (mean $\pm$ standard deviation over five seeds). Missingness $\mu$ is in percent; rates 60--99\% and trapezoidal AUC over $\mu\in[0,0.99]$ (higher is better).}
\label{tab:rt-full-tadpole-high}
\footnotesize\setlength{\tabcolsep}{4pt}
\begin{adjustbox}{max width=\linewidth}
\begin{tabular}{lcccccc}
\toprule
Model & 60 & 70 & 80 & 90 & 99 & AUC \\
\midrule
 GNN-Zero & $0.808\pm0.049$ & $0.788\pm0.083$ & $0.680\pm0.100$ & $0.534\pm0.079$ & $0.429\pm0.049$ & 0.7388 \\
GNN-Median & $0.764\pm0.063$ & $0.714\pm0.138$ & $0.589\pm0.165$ & $0.470\pm0.076$ & $0.267\pm0.057$ & 0.7015 \\
GNN-MI & $0.785\pm0.041$ & $0.749\pm0.101$ & $0.616\pm0.113$ & $0.473\pm0.075$ & $0.241\pm0.000$ & 0.7130 \\
FP & $0.774\pm0.075$ & $0.729\pm0.144$ & $0.676\pm0.120$ & $0.538\pm0.078$ & $0.392\pm0.074$ & 0.7254 \\
PCFI & $0.778\pm0.078$ & $0.766\pm0.093$ & $0.601\pm0.166$ & $0.534\pm0.173$ & $0.435\pm0.067$ & 0.7272 \\
FISF & $0.822\pm0.044$ & $0.711\pm0.146$ & $0.614\pm0.145$ & $0.547\pm0.145$ & $0.380\pm0.137$ & 0.7232 \\
GCNmf & $0.776\pm0.148$ & $0.721\pm0.151$ & $0.601\pm0.149$ & $0.482\pm0.022$ & $0.354\pm0.064$ & 0.7148 \\
 GNN-MIM & $0.823\pm0.045$ & $0.800\pm0.067$ & $0.695\pm0.087$ & $0.566\pm0.109$ & $0.354\pm0.066$ & 0.7469 \\
 PEMix & $0.839\pm0.045$ & $0.844\pm0.039$ & $0.808\pm0.043$ & $0.791\pm0.067$ & $0.822\pm0.047$ & \textbf{0.8243} \\
\bottomrule
\end{tabular}
\end{adjustbox}
\end{table*}

\begin{table*}[p]
\centering
\caption{TADPOLE: complete macro-F1 results under U-MCAR (mean $\pm$ standard deviation over five seeds). Missingness $\mu$ is in percent; rates 0--50\%; AUC is reported in the continuation.}
\label{tab:umcar-full-tadpole}
\footnotesize\setlength{\tabcolsep}{4pt}
\begin{adjustbox}{max width=\linewidth}
\begin{tabular}{lcccccc}
\toprule
Model & 0 & 10 & 20 & 30 & 40 & 50 \\
\midrule
GNN-Zero & $0.825\pm0.038$ & $0.824\pm0.033$ & $0.804\pm0.044$ & $0.804\pm0.027$ & $0.807\pm0.019$ & $0.771\pm0.045$ \\
GNN-Median & $0.825\pm0.038$ & $0.836\pm0.038$ & $0.791\pm0.030$ & $0.795\pm0.051$ & $0.801\pm0.040$ & $0.773\pm0.038$ \\
GNN-MI & $0.825\pm0.038$ & $0.821\pm0.035$ & $0.811\pm0.041$ & $0.798\pm0.032$ & $0.806\pm0.028$ & $0.763\pm0.034$ \\
FP & $0.806\pm0.065$ & $0.836\pm0.038$ & $0.808\pm0.044$ & $0.798\pm0.037$ & $0.784\pm0.026$ & $0.782\pm0.051$ \\
 PCFI & $0.796\pm0.054$ & $0.810\pm0.046$ & $0.811\pm0.038$ & $0.813\pm0.043$ & $0.804\pm0.028$ & $0.800\pm0.035$ \\
FISF & $0.806\pm0.040$ & $0.825\pm0.058$ & $0.811\pm0.036$ & $0.795\pm0.045$ & $0.790\pm0.045$ & $0.797\pm0.035$ \\
 GCNmf & $0.814\pm0.043$ & $0.831\pm0.058$ & $0.792\pm0.036$ & $0.806\pm0.042$ & $0.838\pm0.066$ & $0.812\pm0.019$ \\
GNN-MIM & $0.828\pm0.038$ & $0.838\pm0.044$ & $0.823\pm0.049$ & $0.806\pm0.028$ & $0.809\pm0.029$ & $0.779\pm0.058$ \\
 PEMix & $0.828\pm0.042$ & $0.828\pm0.049$ & $0.819\pm0.011$ & $0.843\pm0.050$ & $0.858\pm0.034$ & $0.826\pm0.041$ \\
\bottomrule
\end{tabular}
\end{adjustbox}
\end{table*}

\begin{table*}[p]
\centering
\caption{TADPOLE: complete macro-F1 results under U-MCAR (mean $\pm$ standard deviation over five seeds). Missingness $\mu$ is in percent; rates 60--99\% and trapezoidal AUC over $\mu\in[0,0.99]$ (higher is better).}
\label{tab:umcar-full-tadpole-high}
\footnotesize\setlength{\tabcolsep}{4pt}
\begin{adjustbox}{max width=\linewidth}
\begin{tabular}{lcccccc}
\toprule
Model & 60 & 70 & 80 & 90 & 99 & AUC \\
\midrule
GNN-Zero & $0.758\pm0.050$ & $0.743\pm0.067$ & $0.723\pm0.070$ & $0.723\pm0.046$ & $0.515\pm0.250$ & 0.7565 \\
GNN-Median & $0.771\pm0.038$ & $0.761\pm0.045$ & $0.738\pm0.036$ & $0.644\pm0.073$ & $0.241\pm0.000$ & 0.7398 \\
GNN-MI & $0.779\pm0.049$ & $0.762\pm0.063$ & $0.744\pm0.072$ & $0.647\pm0.107$ & $0.245\pm0.008$ & 0.7421 \\
FP & $0.788\pm0.044$ & $0.735\pm0.048$ & $0.770\pm0.052$ & $0.695\pm0.060$ & $0.241\pm0.000$ & 0.7473 \\
 PCFI & $0.785\pm0.049$ & $0.774\pm0.036$ & $0.791\pm0.050$ & $0.734\pm0.046$ & $0.241\pm0.000$ & 0.7593 \\
FISF & $0.781\pm0.060$ & $0.747\pm0.079$ & $0.773\pm0.054$ & $0.774\pm0.061$ & $0.241\pm0.000$ & 0.7566 \\
 GCNmf & $0.780\pm0.016$ & $0.736\pm0.051$ & $0.798\pm0.053$ & $0.675\pm0.073$ & $0.326\pm0.057$ & 0.7588 \\
GNN-MIM & $0.761\pm0.067$ & $0.752\pm0.060$ & $0.731\pm0.058$ & $0.696\pm0.041$ & $0.241\pm0.000$ & 0.7482 \\
 PEMix & $0.816\pm0.049$ & $0.796\pm0.041$ & $0.770\pm0.054$ & $0.817\pm0.017$ & $0.775\pm0.057$ & \textbf{0.8094} \\
\bottomrule
\end{tabular}
\end{adjustbox}
\end{table*}

%

\begin{table*}[p]
\centering
\caption{GraphMSO: complete MAE results under RT (five seeds). Missingness $\mu$ is in percent; Rates 0--50\%; AUC is reported in the continuation. Lower is better.}
\label{tab:rt-full-graphmso}
\footnotesize\setlength{\tabcolsep}{4pt}
\begin{adjustbox}{max width=\linewidth}
\begin{tabular}{lcccccc}
\toprule
Model & 0 & 10 & 20 & 30 & 40 & 50 \\
\midrule
GNN-Zero & $1.3637\pm0.0009$ & $1.3732\pm0.0014$ & $1.3775\pm0.0016$ & $1.3803\pm0.0011$ & $1.3835\pm0.0013$ & $1.3866\pm0.0011$ \\
 GNN-MI & $1.3604\pm0.0003$ & $1.3731\pm0.0014$ & $1.3775\pm0.0016$ & $1.3801\pm0.0007$ & $1.3835\pm0.0013$ & $1.3859\pm0.0013$ \\
GNN-Median & $1.3606\pm0.0008$ & $1.3732\pm0.0014$ & $1.3774\pm0.0014$ & $1.3803\pm0.0012$ & $1.3836\pm0.0013$ & $1.3859\pm0.0013$ \\
GNN-MIM & $1.3602\pm0.0007$ & $1.3752\pm0.0020$ & $1.3791\pm0.0017$ & $1.3818\pm0.0013$ & $1.3850\pm0.0018$ & $1.3882\pm0.0009$ \\
FP & $1.3583\pm0.0003$ & $1.3742\pm0.0012$ & $1.3783\pm0.0011$ & $1.3814\pm0.0007$ & $1.3842\pm0.0006$ & $1.3867\pm0.0011$ \\
PCFI & $1.3582\pm0.0005$ & $1.3729\pm0.0013$ & $1.3775\pm0.0011$ & $1.3819\pm0.0004$ & $1.3853\pm0.0010$ & $1.3880\pm0.0013$ \\
FISF & $1.3632\pm0.0006$ & $1.3758\pm0.0012$ & $1.3804\pm0.0013$ & $1.3849\pm0.0009$ & $1.3901\pm0.0013$ & $1.3906\pm0.0013$ \\
 GCNmf & $1.3583\pm0.0006$ & $1.3743\pm0.0013$ & $1.3760\pm0.0009$ & $1.3775\pm0.0006$ & $1.3804\pm0.0012$ & $1.3844\pm0.0003$ \\
 PEMix & $1.3578\pm0.0008$ & $1.3715\pm0.0010$ & $1.3770\pm0.0010$ & $1.3785\pm0.0005$ & $1.3813\pm0.0006$ & $1.3843\pm0.0013$ \\
\bottomrule
\end{tabular}
\end{adjustbox}
\end{table*}

\begin{table*}[p]
\centering
\caption{GraphMSO: complete MAE results under RT (five seeds). Missingness $\mu$ is in percent; Rates 60--99\% and trapezoidal AUC over $\mu\in[0,0.99]$ (lower is better).}
\label{tab:rt-full-graphmso-high}
\footnotesize\setlength{\tabcolsep}{4pt}
\begin{adjustbox}{max width=\linewidth}
\begin{tabular}{lcccccc}
\toprule
Model & 60 & 70 & 80 & 90 & 99 & AUC \\
\midrule
GNN-Zero & $1.3872\pm0.0012$ & $1.3893\pm0.0008$ & $1.3917\pm0.0007$ & $1.3941\pm0.0011$ & $1.3946\pm0.0003$ & 1.3696 \\
 GNN-MI & $1.3872\pm0.0011$ & $1.3894\pm0.0009$ & $1.3918\pm0.0006$ & $1.3942\pm0.0010$ & $1.3949\pm0.0002$ & 1.3696 \\
GNN-Median & $1.3872\pm0.0011$ & $1.3895\pm0.0008$ & $1.3918\pm0.0007$ & $1.3942\pm0.0009$ & $1.3949\pm0.0004$ & 1.3697 \\
GNN-MIM & $1.3912\pm0.0023$ & $1.3907\pm0.0007$ & $1.3935\pm0.0006$ & $1.3946\pm0.0005$ & $1.3950\pm0.0002$ & 1.3712 \\
FP & $1.3894\pm0.0010$ & $1.3925\pm0.0024$ & $1.3948\pm0.0009$ & $1.3977\pm0.0014$ & $1.3964\pm0.0006$ & 1.3717 \\
PCFI & $1.3915\pm0.0006$ & $1.3960\pm0.0011$ & $1.3994\pm0.0009$ & $1.4031\pm0.0019$ & $1.3976\pm0.0006$ & 1.3722 \\
FISF & $1.3943\pm0.0008$ & $1.3975\pm0.0021$ & $1.3976\pm0.0020$ & $1.3981\pm0.0010$ & $1.3967\pm0.0025$ & 1.3738 \\
 GCNmf & $1.3868\pm0.0006$ & $1.3889\pm0.0004$ & $1.3916\pm0.0005$ & $1.3936\pm0.0012$ & $1.3945\pm0.0004$ & 1.3690 \\
 PEMix & $1.3867\pm0.0007$ & $1.3887\pm0.0004$ & $1.3919\pm0.0006$ & $1.3936\pm0.0010$ & $1.3944\pm0.0003$ & \textbf{1.3687} \\
\bottomrule
\end{tabular}
\end{adjustbox}
\end{table*}


\begin{table*}[p]
\centering
\caption{EngRAD: complete MAE results under RT shared-seed missingness
(five seeds, mean $\pm$ sample standard deviation).
Missingness $\mu$ is in percent; rates 0--50\%.
AUC is reported in the continuation. Lower is better.}
\label{tab:rt-full-engrad}
\footnotesize\setlength{\tabcolsep}{4pt}
\begin{adjustbox}{max width=\linewidth}
\begin{tabular}{lcccccc}
\toprule
Model & 0 & 10 & 20 & 30 & 40 & 50 \\
\midrule
GNN-Zero
& $5.4762\pm0.0290$ & $5.8960\pm0.0709$
& $6.3288\pm0.0976$ & $6.7150\pm0.0982$
& $7.0647\pm0.0668$ & $7.4673\pm0.0701$ \\
GNN-MI
& $5.4364\pm0.0528$ & $5.7982\pm0.0400$
& $6.3365\pm0.0596$ & $6.6810\pm0.0875$
& $7.0631\pm0.0963$ & $7.4423\pm0.0760$ \\
GNN-Median
& $5.4321\pm0.0377$ & $5.7998\pm0.0483$
& $6.2912\pm0.0845$ & $6.7284\pm0.1058$
& $7.0481\pm0.0850$ & $7.3866\pm0.0979$ \\
GNN-MIM
& $5.4174\pm0.0395$ & $5.8991\pm0.0407$
& $6.3634\pm0.1025$ & $6.8770\pm0.1471$
& $7.4792\pm0.1344$ & $7.8061\pm0.0907$ \\
 FP
& $5.4181\pm0.0207$ & $5.4721\pm0.0229$
& $5.7357\pm0.1045$ & $5.8706\pm0.0598$
& $6.2543\pm0.1555$ & $6.5903\pm0.2538$ \\
PCFI
& $5.4377\pm0.0635$ & $5.8268\pm0.1653$
& $6.3194\pm0.2407$ & $7.0365\pm0.1464$
& $7.6782\pm0.4432$ & $8.4610\pm0.4943$ \\
 FISF
& $5.4015\pm0.0128$ & $5.5008\pm0.0281$
& $5.7462\pm0.0232$ & $5.8516\pm0.0819$
& $6.0954\pm0.1340$ & $6.3255\pm0.1357$ \\
GCNmf
& $5.6397\pm0.0942$ & $5.9493\pm0.1136$
& $6.4047\pm0.1022$ & $6.7618\pm0.0715$
& $7.1063\pm0.0705$ & $7.5757\pm0.1372$ \\
 PEMix
& $5.3736\pm0.0194$ & $5.7791\pm0.0780$
& $6.1751\pm0.0635$ & $6.7031\pm0.0929$
& $7.0101\pm0.0703$ & $7.1536\pm0.0951$ \\
\bottomrule
\end{tabular}
\end{adjustbox}
\end{table*}

\begin{table*}[p]
\centering
\caption{EngRAD: complete MAE results under RT shared-seed missingness
(five seeds, mean $\pm$ sample standard deviation).
Missingness $\mu$ is in percent; rates 60--99\%.
Trapezoidal AUC is computed over $\mu\in[0,0.99]$
for each seed and reported as mean $\pm$ sample standard deviation.
Lower is better.}
\label{tab:rt-full-engrad-high}
\footnotesize\setlength{\tabcolsep}{4pt}
\resizebox{\textwidth}{!}{%
\begin{adjustbox}{max width=\linewidth}
\begin{tabular}{lcccccc}
\toprule
Model & 60 & 70 & 80 & 90 & 99 & AUC \\
\midrule
GNN-Zero
& $7.8648\pm0.1059$ & $8.0482\pm0.1109$
& $8.4417\pm0.0630$ & $9.0358\pm0.0620$
& $9.5747\pm0.0534$ & $7.345719\pm0.036379$ \\
GNN-MI
& $7.8264\pm0.0940$ & $8.0547\pm0.1006$
& $8.4494\pm0.0871$ & $9.0045\pm0.0743$
& $9.5532\pm0.0396$ & $7.322295\pm0.034537$ \\
GNN-Median
& $7.8062\pm0.1106$ & $8.0017\pm0.1033$
& $8.4592\pm0.0601$ & $9.0617\pm0.0710$
& $9.6248\pm0.0392$ & $7.317701\pm0.046298$ \\
GNN-MIM
& $8.1090\pm0.0876$ & $8.4326\pm0.1660$
& $8.9219\pm0.0680$ & $9.2955\pm0.0545$
& $9.8190\pm0.1070$ & $7.584642\pm0.046850$ \\
 FP
& $6.8111\pm0.1469$ & $6.9421\pm0.1430$
& $7.2574\pm0.1736$ & $7.6992\pm0.1177$
& $9.3221\pm0.0448$ & $6.515191\pm0.097710$ \\
PCFI
& $8.9782\pm0.6702$ & $9.7153\pm0.6322$
& $10.0908\pm0.4348$ & $9.4165\pm0.4403$
& $9.0173\pm0.0836$ & $7.982851\pm0.241905$ \\
 FISF
& $6.6222\pm0.1130$ & $6.7754\pm0.1147$
& $7.1723\pm0.1351$ & $7.7738\pm0.1836$
& $10.0132\pm0.3947$
& $\mathbf{6.468120\pm0.074761}$ \\
GCNmf
& $7.7186\pm0.2210$ & $8.0195\pm0.1684$
& $8.5161\pm0.0955$ & $8.9464\pm0.0412$
& $9.7083\pm0.0460$ & $7.373967\pm0.043182$ \\
 PEMix
& $7.6081\pm0.2022$ & $7.7819\pm0.1405$
& $8.3722\pm0.0844$ & $8.9978\pm0.1453$
& $9.4479\pm0.0578$ & $7.206951\pm0.047572$ \\
\bottomrule
\end{tabular}
\end{adjustbox}}
\end{table*}


\begin{table*}[p]
\centering
\caption{PEMS-BAY: complete MAE results under RT (five seeds). Missingness $\mu$ is in percent; Rates 0--50\%; AUC is reported in the continuation. Lower is better.}
\label{tab:rt-full-pemsbay}
\footnotesize\setlength{\tabcolsep}{4pt}
\begin{adjustbox}{max width=\linewidth}
\begin{tabular}{lcccccc}
\toprule
Model & 0 & 10 & 20 & 30 & 40 & 50 \\
\midrule
GNN-Zero & $3.8743\pm0.0892$ & $3.8786\pm0.0594$ & $3.9786\pm0.0951$ & $4.0962\pm0.1045$ & $4.3490\pm0.1824$ & $4.3858\pm0.3289$ \\
GNN-MI & $3.9270\pm0.1181$ & $3.9083\pm0.0807$ & $3.9255\pm0.0690$ & $4.0772\pm0.0740$ & $4.3852\pm0.1936$ & $4.3559\pm0.1290$ \\
GNN-Median & $3.7836\pm0.0515$ & $3.8937\pm0.0691$ & $4.0551\pm0.1404$ & $4.0427\pm0.0756$ & $4.2222\pm0.0705$ & $4.6794\pm0.4123$ \\
 GNN-MIM & $3.8048\pm0.0902$ & $3.7776\pm0.0335$ & $3.8063\pm0.0525$ & $3.9676\pm0.0963$ & $4.1550\pm0.1220$ & $4.4299\pm0.2171$ \\
FP & $3.7935\pm0.0726$ & $3.8333\pm0.0828$ & $3.9304\pm0.1112$ & $4.0447\pm0.1111$ & $4.0875\pm0.1562$ & $4.2809\pm0.1367$ \\
PCFI & $3.8095\pm0.0597$ & $3.9127\pm0.0745$ & $3.9554\pm0.0530$ & $4.1170\pm0.0442$ & $4.3933\pm0.0864$ & $4.4917\pm0.0734$ \\
FISF & $3.8053\pm0.0601$ & $3.8330\pm0.0324$ & $4.0053\pm0.1443$ & $4.0055\pm0.1086$ & $4.0493\pm0.0758$ & $4.3540\pm0.2148$ \\
 GCNmf & $3.7825\pm0.0618$ & $3.7070\pm0.0604$ & $3.6741\pm0.0860$ & $3.7045\pm0.0971$ & $3.7178\pm0.1043$ & $3.8272\pm0.1811$ \\
 PEMix & $3.7540\pm0.0145$ & $3.7058\pm0.0441$ & $3.6933\pm0.0484$ & $3.6723\pm0.0371$ & $3.6425\pm0.0474$ & $3.7244\pm0.0398$ \\
\bottomrule
\end{tabular}
\end{adjustbox}
\end{table*}

\begin{table*}[p]
\centering
\caption{PEMS-BAY: complete MAE results under RT (five seeds). Missingness $\mu$ is in percent; Rates 60--99\% and trapezoidal AUC over $\mu\in[0,0.99]$ (lower is better).}
\label{tab:rt-full-pemsbay-high}
\footnotesize\setlength{\tabcolsep}{4pt}
\begin{adjustbox}{max width=\linewidth}
\begin{tabular}{lcccccc}
\toprule
Model & 60 & 70 & 80 & 90 & 99 & AUC \\
\midrule
GNN-Zero & $4.5305\pm0.1284$ & $4.7491\pm0.1542$ & $4.6543\pm0.1910$ & $4.5401\pm0.0334$ & $4.6001\pm0.0208$ & 4.2942 \\
GNN-MI & $4.3936\pm0.0560$ & $4.5632\pm0.1432$ & $4.5425\pm0.1582$ & $4.5522\pm0.0445$ & $4.6196\pm0.0056$ & 4.2518 \\
GNN-Median & $4.5532\pm0.1190$ & $4.9191\pm0.2908$ & $4.6654\pm0.1563$ & $4.5790\pm0.0628$ & $4.7467\pm0.2893$ & 4.3408 \\
 GNN-MIM & $4.4871\pm0.1917$ & $4.4527\pm0.1908$ & $4.5787\pm0.0771$ & $4.5532\pm0.1272$ & $4.5577\pm0.0298$ & 4.1934 \\
FP & $4.3485\pm0.1574$ & $4.5187\pm0.2141$ & $4.5679\pm0.0786$ & $4.6860\pm0.1439$ & $4.6140\pm0.0440$ & 4.2037 \\
PCFI & $4.5410\pm0.1166$ & $4.6474\pm0.0918$ & $4.5403\pm0.0502$ & $4.7431\pm0.0705$ & $4.6168\pm0.0224$ & 4.3087 \\
FISF & $4.4999\pm0.0960$ & $4.5837\pm0.1526$ & $4.6479\pm0.0790$ & $4.6430\pm0.0494$ & $4.6346\pm0.0639$ & 4.2378 \\
 GCNmf & $3.8703\pm0.1338$ & $4.1159\pm0.1294$ & $4.3221\pm0.1160$ & $4.5483\pm0.1632$ & $4.5958\pm0.0562$ & 3.9219 \\
 PEMix & $3.7331\pm0.0510$ & $3.8074\pm0.0430$ & $3.9609\pm0.0684$ & $4.0777\pm0.1368$ & $4.3424\pm0.0328$ & \textbf{3.7645} \\
\bottomrule
\end{tabular}
\end{adjustbox}
\end{table*}


\begin{table*}[p]
\centering
\caption{EngRAD: complete MAE results under U-MCAR (mean $\pm$ standard deviation over five seeds). Missingness $\mu$ is in percent; rates 0--50\%; AUC is reported in the continuation. Lower is better.}
\label{tab:umcar-full-engrad}
\footnotesize\setlength{\tabcolsep}{4pt}
\begin{adjustbox}{max width=\linewidth}
\begin{tabular}{lcccccc}
\toprule
Model & 0 & 10 & 20 & 30 & 40 & 50 \\
\midrule
 GNN-Zero & $5.462\pm0.038$ & $5.491\pm0.047$ & $5.560\pm0.051$ & $5.610\pm0.026$ & $5.686\pm0.050$ & $5.721\pm0.086$ \\
GNN-Median & $5.412\pm0.030$ & $5.496\pm0.030$ & $5.621\pm0.046$ & $5.668\pm0.066$ & $5.773\pm0.049$ & $5.840\pm0.036$ \\
GNN-MI & $5.395\pm0.037$ & $5.509\pm0.029$ & $5.589\pm0.055$ & $5.649\pm0.049$ & $5.675\pm0.046$ & $5.721\pm0.030$ \\
 FP & $5.417\pm0.028$ & $5.404\pm0.023$ & $5.434\pm0.013$ & $5.444\pm0.021$ & $5.459\pm0.051$ & $5.429\pm0.022$ \\
PCFI & $5.416\pm0.029$ & $5.604\pm0.032$ & $5.771\pm0.026$ & $5.996\pm0.042$ & $6.062\pm0.124$ & $6.303\pm0.090$ \\
 FISF & $5.439\pm0.053$ & $5.416\pm0.046$ & $5.453\pm0.049$ & $5.431\pm0.035$ & $5.473\pm0.049$ & $5.499\pm0.056$ \\
GCNmf & $5.616\pm0.040$ & $5.635\pm0.061$ & $5.895\pm0.059$ & $5.895\pm0.057$ & $5.994\pm0.041$ & $5.816\pm0.047$ \\
GNN-MIM & $5.419\pm0.049$ & $5.536\pm0.030$ & $5.607\pm0.051$ & $5.683\pm0.109$ & $5.723\pm0.056$ & $5.784\pm0.093$ \\
PEMix & $5.370 \pm 0.024$ & $5.507 \pm 0.046$ & $5.515 \pm 0.019$ & $5.564 \pm 0.035$ & $5.604 \pm 0.028$ & $5.662 \pm 0.034$ \\
\bottomrule
\end{tabular}
\end{adjustbox}
\end{table*}

\begin{table*}[p]
\centering
\caption{EngRAD: complete MAE results under U-MCAR (mean $\pm$ standard deviation over five seeds). Missingness $\mu$ is in percent; rates 60--99\% and trapezoidal AUC over $\mu\in[0,0.99]$ (lower is better).}
\label{tab:umcar-full-engrad-high}
\footnotesize\setlength{\tabcolsep}{4pt}
\begin{adjustbox}{max width=\linewidth}
\begin{tabular}{lcccccc}
\toprule
Model & 60 & 70 & 80 & 90 & 99 & AUC \\
\midrule
 GNN-Zero & $5.842\pm0.107$ & $5.910\pm0.036$ & $5.995\pm0.044$ & $6.277\pm0.017$ & $7.687\pm0.076$ & 5.7968 \\
GNN-Median & $5.904\pm0.046$ & $5.989\pm0.056$ & $6.095\pm0.040$ & $6.328\pm0.041$ & $7.769\pm0.046$ & 5.8600 \\
GNN-MI & $5.868\pm0.115$ & $5.929\pm0.023$ & $6.111\pm0.087$ & $6.265\pm0.025$ & $7.696\pm0.066$ & 5.8164 \\
 FP & $5.472\pm0.028$ & $5.487\pm0.034$ & $5.545\pm0.031$ & $5.770\pm0.087$ & $7.065\pm0.176$ & \textbf{5.5043} \\
PCFI & $6.408\pm0.058$ & $6.645\pm0.079$ & $6.905\pm0.046$ & $7.407\pm0.080$ & $9.586\pm0.077$ & 6.3751 \\
 FISF & $5.558\pm0.028$ & $5.637\pm0.062$ & $5.707\pm0.046$ & $6.113\pm0.077$ & $7.422\pm0.053$ & 5.6040 \\
GCNmf & $6.142\pm0.067$ & $6.107\pm0.079$ & $6.213\pm0.124$ & $6.786\pm0.114$ & $8.117\pm0.099$ & 6.0604 \\
GNN-MIM & $5.801\pm0.058$ & $5.919\pm0.200$ & $6.066\pm0.091$ & $6.334\pm0.126$ & $7.870\pm0.086$ & 5.8387 \\
PEMix & $5.772 \pm 0.055$ & $5.861 \pm 0.042$ & $5.980 \pm 0.036$ & $6.224 \pm 0.029$ & $7.714 \pm 0.164$ & 5.7534 \\
\bottomrule
\end{tabular}
\end{adjustbox}
\end{table*}


\begin{table*}[p]
\centering
\caption{GraphMSO: complete MAE results under U-MCAR (mean $\pm$ standard deviation over five seeds). Missingness $\mu$ is in percent; rates 0--50\%; AUC is reported in the continuation. Lower is better.}
\label{tab:umcar-full-graphmso}
\footnotesize\setlength{\tabcolsep}{4pt}
\begin{adjustbox}{max width=\linewidth}
\begin{tabular}{lcccccc}
\toprule
Model & 0 & 10 & 20 & 30 & 40 & 50 \\
\midrule
 GNN-Zero & $1.360\pm0.000$ & $1.373\pm0.001$ & $1.377\pm0.001$ & $1.379\pm0.001$ & $1.381\pm0.001$ & $1.383\pm0.001$ \\
GNN-Median & $1.360\pm0.001$ & $1.373\pm0.001$ & $1.377\pm0.001$ & $1.379\pm0.001$ & $1.381\pm0.001$ & $1.383\pm0.001$ \\
GNN-MI & $1.363\pm0.001$ & $1.373\pm0.001$ & $1.377\pm0.001$ & $1.379\pm0.001$ & $1.381\pm0.001$ & $1.383\pm0.001$ \\
FP & $1.358\pm0.001$ & $1.374\pm0.001$ & $1.378\pm0.002$ & $1.380\pm0.001$ & $1.383\pm0.001$ & $1.385\pm0.001$ \\
PCFI & $1.358\pm0.001$ & $1.373\pm0.001$ & $1.377\pm0.001$ & $1.379\pm0.001$ & $1.382\pm0.001$ & $1.384\pm0.001$ \\
FISF & $1.363\pm0.001$ & $1.375\pm0.001$ & $1.379\pm0.001$ & $1.381\pm0.001$ & $1.384\pm0.001$ & $1.387\pm0.002$ \\
 GCNmf & $1.358\pm0.001$ & $1.372\pm0.001$ & $1.374\pm0.001$ & $1.376\pm0.001$ & $1.379\pm0.001$ & $1.381\pm0.002$ \\
GNN-MIM & $1.361\pm0.001$ & $1.375\pm0.001$ & $1.378\pm0.001$ & $1.381\pm0.001$ & $1.383\pm0.001$ & $1.384\pm0.001$ \\
 PEMix & $1.358\pm0.001$ & $1.371\pm0.002$ & $1.374\pm0.001$ & $1.376\pm0.002$ & $1.378\pm0.000$ & $1.382\pm0.001$ \\
\bottomrule
\end{tabular}
\end{adjustbox}
\end{table*}

\begin{table*}[p]
\centering
\caption{GraphMSO: complete MAE results under U-MCAR (mean $\pm$ standard deviation over five seeds). Missingness $\mu$ is in percent; rates 60--99\% and trapezoidal AUC over $\mu\in[0,0.99]$ (lower is better).}
\label{tab:umcar-full-graphmso-high}
\footnotesize\setlength{\tabcolsep}{4pt}
\begin{adjustbox}{max width=\linewidth}
\begin{tabular}{lcccccc}
\toprule
Model & 60 & 70 & 80 & 90 & 99 & AUC \\
\midrule
 GNN-Zero & $1.385\pm0.001$ & $1.387\pm0.001$ & $1.389\pm0.001$ & $1.392\pm0.001$ & $1.396\pm0.000$ & 1.3684 \\
GNN-Median & $1.385\pm0.001$ & $1.387\pm0.001$ & $1.389\pm0.001$ & $1.392\pm0.001$ & $1.396\pm0.000$ & 1.3684 \\
GNN-MI & $1.385\pm0.001$ & $1.387\pm0.001$ & $1.389\pm0.001$ & $1.392\pm0.001$ & $1.396\pm0.001$ & 1.3686 \\
FP & $1.387\pm0.002$ & $1.389\pm0.001$ & $1.393\pm0.001$ & $1.395\pm0.002$ & $1.398\pm0.001$ & 1.3703 \\
PCFI & $1.386\pm0.002$ & $1.388\pm0.001$ & $1.393\pm0.002$ & $1.397\pm0.002$ & $1.406\pm0.002$ & 1.3701 \\
FISF & $1.388\pm0.002$ & $1.392\pm0.001$ & $1.396\pm0.002$ & $1.401\pm0.002$ & $1.403\pm0.002$ & 1.3726 \\
 GCNmf & $1.384\pm0.001$ & $1.387\pm0.001$ & $1.389\pm0.001$ & $1.392\pm0.001$ & $1.395\pm0.000$ & 1.3671 \\
GNN-MIM & $1.385\pm0.002$ & $1.388\pm0.001$ & $1.390\pm0.000$ & $1.393\pm0.001$ & $1.396\pm0.000$ & 1.3696 \\
 PEMix & $1.384\pm0.001$ & $1.386\pm0.001$ & $1.389\pm0.001$ & $1.392\pm0.001$ & $1.396\pm0.001$ & \textbf{1.3669} \\
\bottomrule
\end{tabular}
\end{adjustbox}
\end{table*}


\begin{table*}[p]
\centering
\caption{PEMS-BAY: complete MAE results under U-MCAR (mean $\pm$ standard deviation over five seeds). Missingness $\mu$ is in percent; rates 0--50\%; AUC is reported in the continuation. Lower is better.}
\label{tab:umcar-full-pems-bay}
\footnotesize\setlength{\tabcolsep}{4pt}
\begin{adjustbox}{max width=\linewidth}
\begin{tabular}{lcccccc}
\toprule
Model & 0 & 10 & 20 & 30 & 40 & 50 \\
\midrule
GNN-Zero & $3.852\pm0.074$ & $3.901\pm0.070$ & $3.882\pm0.058$ & $4.060\pm0.132$ & $4.184\pm0.194$ & $4.201\pm0.088$ \\
GNN-Median & $3.832\pm0.049$ & $3.833\pm0.045$ & $3.907\pm0.034$ & $3.937\pm0.042$ & $4.030\pm0.123$ & $4.081\pm0.037$ \\
GNN-MI & $3.895\pm0.095$ & $3.801\pm0.031$ & $3.914\pm0.141$ & $4.068\pm0.084$ & $4.088\pm0.062$ & $4.262\pm0.074$ \\
FP & $3.799\pm0.079$ & $3.862\pm0.024$ & $3.968\pm0.040$ & $3.925\pm0.049$ & $3.975\pm0.054$ & $3.982\pm0.035$ \\
PCFI & $3.797\pm0.075$ & $3.905\pm0.093$ & $3.981\pm0.096$ & $4.021\pm0.141$ & $4.053\pm0.065$ & $4.089\pm0.041$ \\
FISF & $3.814\pm0.053$ & $3.804\pm0.020$ & $3.922\pm0.024$ & $3.857\pm0.021$ & $3.969\pm0.037$ & $4.041\pm0.063$ \\
 GCNmf & $3.776\pm0.068$ & $3.316\pm0.058$ & $3.278\pm0.027$ & $3.266\pm0.026$ & $3.294\pm0.034$ & $3.279\pm0.025$ \\
 GNN-MIM & $3.830\pm0.065$ & $3.815\pm0.073$ & $3.894\pm0.103$ & $3.928\pm0.029$ & $3.849\pm0.041$ & $3.933\pm0.046$ \\
 PEMix & $3.753\pm0.010$ & $3.424\pm0.027$ & $3.395\pm0.021$ & $3.357\pm0.021$ & $3.390\pm0.024$ & $3.435\pm0.051$ \\
\bottomrule
\end{tabular}
\end{adjustbox}
\end{table*}

\begin{table*}[p]
\centering
\caption{PEMS-BAY: complete MAE results under U-MCAR (mean $\pm$ standard deviation over five seeds). Missingness $\mu$ is in percent; rates 60--99\% and trapezoidal AUC over $\mu\in[0,0.99]$ (lower is better).}
\label{tab:umcar-full-pems-bay-high}
\footnotesize\setlength{\tabcolsep}{4pt}
\begin{adjustbox}{max width=\linewidth}
\begin{tabular}{lcccccc}
\toprule
Model & 60 & 70 & 80 & 90 & 99 & AUC \\
\midrule
GNN-Zero & $4.155\pm0.058$ & $4.104\pm0.135$ & $4.124\pm0.131$ & $4.215\pm0.166$ & $4.574\pm0.049$ & 4.0600 \\
GNN-Median & $4.248\pm0.115$ & $4.037\pm0.091$ & $4.220\pm0.087$ & $4.386\pm0.086$ & $4.474\pm0.068$ & 4.0389 \\
GNN-MI & $4.101\pm0.084$ & $4.300\pm0.056$ & $4.359\pm0.153$ & $4.138\pm0.120$ & $4.500\pm0.059$ & 4.0795 \\
FP & $4.108\pm0.102$ & $4.148\pm0.055$ & $4.290\pm0.078$ & $4.380\pm0.057$ & $4.527\pm0.012$ & 4.0356 \\
PCFI & $4.154\pm0.061$ & $4.260\pm0.097$ & $4.395\pm0.067$ & $4.518\pm0.108$ & $4.677\pm0.047$ & 4.1154 \\
FISF & $4.075\pm0.035$ & $4.228\pm0.064$ & $4.414\pm0.112$ & $4.608\pm0.167$ & $4.708\pm0.055$ & 4.0713 \\
 GCNmf & $3.333\pm0.038$ & $3.375\pm0.011$ & $3.497\pm0.021$ & $3.656\pm0.038$ & $4.299\pm0.054$ & \textbf{3.3935} \\
 GNN-MIM & $3.956\pm0.030$ & $3.977\pm0.048$ & $4.004\pm0.066$ & $4.047\pm0.020$ & $4.386\pm0.058$ & 3.9091 \\
 PEMix & $3.437\pm0.029$ & $3.477\pm0.052$ & $3.564\pm0.076$ & $3.672\pm0.094$ & $3.988\pm0.046$ & 3.4638 \\
\bottomrule
\end{tabular}
\end{adjustbox}
\end{table*}

\clearpage

\section{Implementation and Reproducibility Details}

\label{app:implementation}
This appendix records optimization, covariance assumptions, hyperparameter ranges, and computational complexity. These details are separated from the conceptual method in \Cref{sec:pemix} so the main text follows the mechanism-to-model argument.

\paragraph{Fully-missing nodes.}
When $\mathcal{O}_i=\emptyset$, the feature contribution to the responsibility
vanishes, while the positional term remains active. Consequently, the
imputation can remain node-specific when the PE-conditioned responsibilities distinguish nodes.

\paragraph{Learning the mixture.}
The mixture parameters $(\pi_k,\mu_k,\sigma_k^2)$ are \emph{not} fixed by a
separate EM fit; they are trainable and optimised end-to-end with the GCN
weights $W^{(1)},W^{(2)}$ by minimising the cross-entropy of the
node-classification loss. We initialise them with a few EM iterations on the
mean-imputed extended matrix $Z=[\widetilde X\mid\PE]$, which gives a sensible
starting partition, and then let backpropagation refine them jointly with the
classifier. The mixing weights are kept on the simplex by an unconstrained
softmax re-parameterisation, and the variances are kept positive by optimising
$\log\sigma_k^2$.

\paragraph{Diagonal covariances.}
We use diagonal covariances $\mathrm{diag}(\sigma_k^2)$ throughout. This is
what makes the marginalisation of the missing coordinates closed-form and the
expected ReLU in the probabilistic first layer analytic; full covariances would couple observed and missing coordinates and require more expensive conditional-moment computations, although Gaussian conditioning and scalar expected ReLUs remain analytic. The diagonal assumption is the
main modelling limitation of \GCNmfPE.

\paragraph{Hyperparameters.}
For all competing methods, we follow the hyperparameter search spaces and model-selection
procedures of~\cite{ferrini2026rethinking} and of the original methods, using
the official implementations when available. Model selection is performed exclusively on the
validation set, without access to test performance.

For the controlled-missingness experiments, hyperparameter selection is performed independently
for RT and U-MCAR. Within each mechanism, candidate configurations are evaluated over the full
missingness curve and selected according to validation AUC. The selected network and optimization hyperparameters are reported for
TADPOLE, EngRAD, GraphMSO, and PEMS-BAY in
Tables~\ref{tab:hyperparameters_tadpole},
\ref{tab:hyperparameters_engrad},
\ref{tab:hyperparameters_graphmso}, and
\ref{tab:hyperparameters_pems_bay}, respectively.

For datasets with naturally occurring missingness, no additional missingness is injected and
hyperparameters are selected directly according to validation MAE. The selected configurations
for AQI, PV-US, and METR-LA are reported in
Tables~\ref{tab:hyperparameters_aqi_natural_1},
\ref{tab:hyperparameters_pv_us_natural_1}, and
\ref{tab:hyperparameters_metr_la_natural_1}, respectively.

For PEMix and GCNmf, the GMM-specific hyperparameters are reported separately. In the TADPOLE
experiments, PEMix uses $K=5$ mixture components and PE dimension $q=8$, fixed across both
missingness mechanisms and all missingness rates. For EngRAD, GraphMSO, and PEMS-BAY, the
selected GMM configurations are reported in
Tables~\ref{tab:hyperparameters_engrad_gmm},
\ref{tab:hyperparameters_graphmso_gmm}, and
\ref{tab:hyperparameters_pems_bay_gmm}, respectively. The corresponding settings for AQI,
PV-US, and METR-LA are reported in
Tables~\ref{tab:hyperparameters_aqi_gmm},
\ref{tab:hyperparameters_pv_us_gmm}, and
\ref{tab:hyperparameters_metr_la_gmm}.

\begin{table*}[t]
\centering
\caption{Hyperparameter configurations selected for TADPOLE. Hyperparameters are selected separately for RT and U-MCAR using validation performance, and the selected configuration is then kept fixed across all missingness rates within each mechanism.}
\label{tab:hyperparameters_tadpole}
\small
\setlength{\tabcolsep}{4pt}
\resizebox{\textwidth}{!}{%
\begin{tabular}{llcccccc}
\toprule
Mechanism & Method & Hidden & Dropout & LR & Weight decay & Max epochs & Patience \\
\midrule
RT & GNN-Zero & 128 & 0.5 & 0.01 & 0 & 500 & 50 \\
RT & GNN-Median & 128 & 0.5 & 0.01 & 0 & 500 & 50 \\
RT & GNN-MI & 128 & 0.5 & 0.01 & 0 & 500 & 50 \\
RT & GNN-MIM & 128 & 0.5 & 0.01 & 0 & 500 & 50 \\
RT & FP & 64 & 0.5 & 0.01 & 0 & 1000 & 40 \\
RT & PCFI & 64 & 0.5 & 0.01 & 0 & 1000 & 40 \\
RT & FISF & 64 & 0.5 & 0.01 & 0 & 1000 & 40 \\
RT & GCNmf & 16 & 0 & 0.01 & $10^{-5}$ & 1000 & 40 \\
RT & PEMix & 16 & 0 & 0.005 & 0 & 500 & 50 \\
\midrule
U-MCAR & GNN-Zero & 128 & 0.5 & 0.01 & 0 & 500 & 50 \\
U-MCAR & GNN-Median & 128 & 0.5 & 0.01 & 0 & 500 & 50 \\
U-MCAR & GNN-MI & 128 & 0.5 & 0.01 & 0 & 500 & 50 \\
U-MCAR & GNN-MIM & 128 & 0.5 & 0.01 & 0 & 500 & 50 \\
U-MCAR & FP & 64 & 0.5 & 0.01 & 0 & 1000 & 40 \\
U-MCAR & PCFI & 64 & 0.5 & 0.01 & 0 & 1000 & 40 \\
U-MCAR & FISF & 64 & 0.5 & 0.01 & 0 & 1000 & 40 \\
U-MCAR & GCNmf & 16 & 0 & 0.01 & $10^{-5}$ & 1000 & 40 \\
U-MCAR & PEMix & 16 & 0 & 0.005 & 0 & 500 & 50 \\
\bottomrule
\end{tabular}%
}
\end{table*}

\begin{table*}[t]
\centering
\caption{Hyperparameter configurations selected for EngRAD. Hyperparameters are selected separately for RT and U-MCAR using validation performance, and the selected configuration is then kept fixed across all missingness rates within each mechanism.}
\label{tab:hyperparameters_engrad}
\small
\setlength{\tabcolsep}{4pt}
\resizebox{\textwidth}{!}{%
\begin{tabular}{llcccccc}
\toprule
Mechanism & Method & Hidden & Dropout & LR & Weight decay & Max epochs & Patience \\
\midrule
RT & GNN-Zero & 128 & 0.5 & 0.001 & 0 & 500 & 50 \\
RT & GNN-Median & 128 & 0.3 & 0.001 & 0 & 500 & 50 \\
RT & GNN-MI & 128 & 0.5 & 0.001 & 0 & 500 & 50 \\
RT & GNN-MIM & 128 & 0.3 & 0.001 & 0 & 500 & 50 \\
RT & FP & 64 & 0.3 & 0.001 & 0 & 1000 & 40 \\
RT & PCFI & 64 & 0.5 & 0.001 & 0 & 1000 & 40 \\
RT & FISF & 64 & 0.5 & 0.01 & 0 & 1000 & 40 \\
RT & GCNmf & 16 & 0.5 & 0.003 & 0 & 1000 & 40 \\
RT & PEMix & 32 & 0.3 & 0.001 & 0 & 500 & 50 \\
\midrule
U-MCAR & GNN-Zero & 128 & 0.5 & 0.001 & 0 & 500 & 50 \\
U-MCAR & GNN-Median & 128 & 0.3 & 0.001 & 0 & 500 & 50 \\
U-MCAR & GNN-MI & 128 & 0.5 & 0.001 & 0 & 500 & 50 \\
U-MCAR & GNN-MIM & 128 & 0.1 & 0.001 & 0 & 500 & 50 \\
U-MCAR & FP & 64 & 0.5 & 0.003 & 0 & 1000 & 40 \\
U-MCAR & PCFI & 64 & 0.5 & 0.001 & 0 & 1000 & 40 \\
U-MCAR & FISF & 64 & 0.5 & 0.001 & 0 & 1000 & 40 \\
U-MCAR & GCNmf & 16 & 0.5 & 0.003 & 0 & 1000 & 40 \\
U-MCAR & PEMix & 16 & 0.5 & 0.001 & 0 & 500 & 50 \\
\bottomrule
\end{tabular}%
}
\end{table*}

\begin{table*}[t]
\centering
\caption{Hyperparameter configurations selected for GraphMSO. Hyperparameters are selected separately for RT and U-MCAR using validation performance, and the selected configuration is then kept fixed across all missingness rates within each mechanism.}
\label{tab:hyperparameters_graphmso}
\small
\setlength{\tabcolsep}{4pt}
\resizebox{\textwidth}{!}{%
\begin{tabular}{llcccccc}
\toprule
Mechanism & Method & Hidden & Dropout & LR & Weight decay & Max epochs & Patience \\
\midrule
RT & GNN-Zero & 128 & 0.5 & 0.003 & 0 & 500 & 50 \\
RT & GNN-Median & 128 & 0.5 & 0.01 & 0 & 500 & 50 \\
RT & GNN-MI & 128 & 0.5 & 0.003 & 0 & 500 & 50 \\
RT & GNN-MIM & 128 & 0.1 & 0.003 & 0 & 500 & 50 \\
RT & FP & 64 & 0.5 & 0.01 & 0 & 1000 & 40 \\
RT & PCFI & 64 & 0.5 & 0.01 & 0 & 1000 & 40 \\
RT & FISF & 64 & 0.5 & 0.01 & 0 & 1000 & 40 \\
RT & GCNmf & 16 & 0.5 & 0.01 & 0 & 1000 & 40 \\
RT & PEMix & 16 & 0.5 & 0.01 & 0 & 500 & 50 \\
\midrule
U-MCAR & GNN-Zero & 128 & 0.5 & 0.003 & 0 & 500 & 50 \\
U-MCAR & GNN-Median & 128 & 0.5 & 0.01 & 0 & 500 & 50 \\
U-MCAR & GNN-MI & 128 & 0.5 & 0.01 & 0 & 500 & 50 \\
U-MCAR & GNN-MIM & 128 & 0.3 & 0.003 & 0 & 500 & 50 \\
U-MCAR & FP & 64 & 0.5 & 0.01 & 0 & 1000 & 40 \\
U-MCAR & PCFI & 64 & 0.5 & 0.01 & 0 & 1000 & 40 \\
U-MCAR & FISF & 64 & 0.5 & 0.01 & 0 & 1000 & 40 \\
U-MCAR & GCNmf & 16 & 0 & 0.01 & 0 & 1000 & 40 \\
U-MCAR & PEMix & 16 & 0.5 & 0.01 & 0 & 500 & 50 \\
\bottomrule
\end{tabular}%
}
\end{table*}

\begin{table*}[t]
\centering
\caption{Hyperparameter configurations selected for PEMS-BAY. Hyperparameters are selected separately for RT and U-MCAR using validation performance, and the selected configuration is then kept fixed across all missingness rates within each mechanism.}
\label{tab:hyperparameters_pems_bay}
\small
\setlength{\tabcolsep}{4pt}
\resizebox{\textwidth}{!}{%
\begin{tabular}{llcccccc}
\toprule
Mechanism & Method & Hidden & Dropout & LR & Weight decay & Max epochs & Patience \\
\midrule
RT & GNN-Zero & 128 & 0 & 0.001 & 0 & 500 & 50 \\
RT & GNN-Median & 128 & 0 & 0.003 & 0 & 500 & 50 \\
RT & GNN-MI & 128 & 0.3 & 0.001 & 0 & 500 & 50 \\
RT & GNN-MIM & 128 & 0 & 0.001 & 0 & 500 & 50 \\
RT & FP & 64 & 0.3 & 0.003 & 0 & 1000 & 40 \\
RT & PCFI & 64 & 0.1 & 0.003 & 0 & 1000 & 40 \\
RT & FISF & 64 & 0.5 & 0.001 & 0 & 1000 & 40 \\
RT & GCNmf & 16 & 0 & 0.01 & 0 & 1000 & 40 \\
RT & PEMix & 16 & 0.5 & 0.01 & 0 & 500 & 50 \\
\midrule
U-MCAR & GNN-Zero & 128 & 0.5 & 0.003 & 0 & 500 & 50 \\
U-MCAR & GNN-Median & 128 & 0.1 & 0.001 & 0 & 500 & 50 \\
U-MCAR & GNN-MI & 128 & 0.5 & 0.003 & 0 & 500 & 50 \\
U-MCAR & GNN-MIM & 128 & 0 & 0.001 & 0 & 500 & 50 \\
U-MCAR & FP & 64 & 0 & 0.003 & 0 & 1000 & 40 \\
U-MCAR & PCFI & 64 & 0.1 & 0.001 & 0 & 1000 & 40 \\
U-MCAR & FISF & 64 & 0.1 & 0.001 & 0 & 1000 & 40 \\
U-MCAR & GCNmf & 16 & 0 & 0.01 & 0 & 1000 & 40 \\
U-MCAR & PEMix & 16 & 0.5 & 0.001 & 0 & 500 & 50 \\
\bottomrule
\end{tabular}%
}
\end{table*}

\begin{table*}[t]
\centering
\caption{Hyperparameter configurations selected for AQI under natural missingness. Configurations are selected using validation MAE; no additional missingness is injected. GMM settings are reported separately.}
\label{tab:hyperparameters_aqi_natural_1}
\small
\setlength{\tabcolsep}{4pt}
\resizebox{\textwidth}{!}{%
\begin{tabular}{lcccccc}
\toprule
Method & Hidden & Dropout & LR & Weight decay & Max epochs & Patience \\
\midrule
GNN-Zero & 128 & 0.1 & 0.01 & 0 & 500 & 50 \\
\addlinespace[2pt]
GNN-Median & 128 & 0.3 & 0.01 & 0 & 500 & 50 \\
\addlinespace[2pt]
GNN-MI & 128 & 0.5 & 0.01 & 0 & 500 & 50 \\
\addlinespace[2pt]
GNN-MIM & 128 & 0.5 & 0.01 & 0 & 500 & 50 \\
\addlinespace[2pt]
FP & 64 & 0.3 & 0.001 & 0 & 1000 & 40 \\
\addlinespace[2pt]
PCFI & 64 & 0.5 & 0.003 & 0 & 1000 & 40 \\
\addlinespace[2pt]
FISF & 64 & 0.1 & 0.01 & 0 & 1000 & 40 \\
\addlinespace[2pt]
GCNmf & 16 & 0 & 0.01 & 0 & 1000 & 40 \\
\addlinespace[2pt]
PEMix & 16 & 0.5 & 0.003 & 0 & 500 & 50 \\
\bottomrule
\end{tabular}%
}
\end{table*}

\begin{table*}[t]
\centering
\caption{Hyperparameter configurations selected for PV-US under natural missingness. Configurations are selected using validation MAE; no additional missingness is injected. GMM settings are reported separately.}
\label{tab:hyperparameters_pv_us_natural_1}
\small
\setlength{\tabcolsep}{4pt}
\resizebox{\textwidth}{!}{%
\begin{tabular}{lcccccc}
\toprule
Method & Hidden & Dropout & LR & Weight decay & Max epochs & Patience \\
\midrule
GNN-Zero & 128 & 0.5 & 0.003 & 0 & 500 & 50 \\
\addlinespace[2pt]
GNN-Median & 128 & 0 & 0.003 & 0 & 500 & 50 \\
\addlinespace[2pt]
GNN-MI & 128 & 0.3 & 0.003 & 0 & 500 & 50 \\
\addlinespace[2pt]
GNN-MIM & 128 & 0.5 & 0.001 & 0 & 500 & 50 \\
\addlinespace[2pt]
FP & 64 & 0.3 & 0.003 & 0 & 1000 & 40 \\
\addlinespace[2pt]
PCFI & 64 & 0.5 & 0.003 & 0 & 1000 & 40 \\
\addlinespace[2pt]
FISF & 64 & 0.3 & 0.001 & 0 & 1000 & 40 \\
\addlinespace[2pt]
GCNmf & 16 & 0 & 0.01 & 0 & 1000 & 40 \\
\addlinespace[2pt]
PEMix & 16 & 0.5 & 0.01 & 0 & 500 & 50 \\
\bottomrule
\end{tabular}%
}
\end{table*}

\begin{table*}[t]
\centering
\caption{Hyperparameter configurations selected for METR-LA under natural missingness. Configurations are selected using validation MAE; no additional missingness is injected. GMM settings are reported separately.}
\label{tab:hyperparameters_metr_la_natural_1}
\small
\setlength{\tabcolsep}{4pt}
\resizebox{\textwidth}{!}{%
\begin{tabular}{lcccccc}
\toprule
Method & Hidden & Dropout & LR & Weight decay & Max epochs & Patience \\
\midrule
GNN-Zero & 128 & 0 & 0.001 & 0 & 500 & 50 \\
\addlinespace[2pt]
GNN-Median & 128 & 0.1 & 0.001 & 0 & 500 & 50 \\
\addlinespace[2pt]
GNN-MI & 128 & 0 & 0.001 & 0 & 500 & 50 \\
\addlinespace[2pt]
GNN-MIM & 128 & 0.3 & 0.003 & 0 & 500 & 50 \\
\addlinespace[2pt]
FP & 64 & 0 & 0.003 & 0 & 1000 & 40 \\
\addlinespace[2pt]
PCFI & 64 & 0 & 0.003 & 0 & 1000 & 40 \\
\addlinespace[2pt]
FISF & 64 & 0 & 0.003 & 0 & 1000 & 40 \\
\addlinespace[2pt]
GCNmf & 16 & 0 & 0.001 & 0 & 1000 & 40 \\
\addlinespace[2pt]
PEMix & 16 & 0.5 & 0.01 & 0 & 500 & 50 \\
\bottomrule
\end{tabular}%
}
\end{table*}


\begin{table*}[t]
\centering
\caption{GMM settings selected for EngRAD. $K$ is the number of mixture components and $q$ is the PE dimension; GCNmf does not use PE ($q=0$).}
\label{tab:hyperparameters_engrad_gmm}
\small
\setlength{\tabcolsep}{7pt}
\begin{tabular}{llccl}
\toprule
Mechanism & Method & $K$ & $q$ & Likelihood \\
\midrule
RT & GCNmf & 5 & 0 & legacy \\
RT & PEMix & 8 & 16 & gaussian \\
\midrule
U-MCAR & GCNmf & 5 & 0 & legacy \\
U-MCAR & PEMix & 5 & 8 & gaussian \\
\bottomrule
\end{tabular}
\end{table*}

\begin{table*}[t]
\centering
\caption{GMM settings selected for GraphMSO. $K$ is the number of mixture components and $q$ is the PE dimension; GCNmf does not use PE ($q=0$).}
\label{tab:hyperparameters_graphmso_gmm}
\small
\setlength{\tabcolsep}{7pt}
\begin{tabular}{llccl}
\toprule
Mechanism & Method & $K$ & $q$ & Likelihood \\
\midrule
RT & GCNmf & 5 & 0 & legacy \\
RT & PEMix & 3 & 16 & gaussian \\
\midrule
U-MCAR & GCNmf & 5 & 0 & gaussian \\
U-MCAR & PEMix & 2 & 8 & gaussian \\
\bottomrule
\end{tabular}
\end{table*}

\begin{table*}[t]
\centering
\caption{GMM settings selected for PEMS-BAY. $K$ is the number of mixture components and $q$ is the PE dimension; GCNmf does not use PE ($q=0$).}
\label{tab:hyperparameters_pems_bay_gmm}
\small
\setlength{\tabcolsep}{7pt}
\begin{tabular}{llccl}
\toprule
Mechanism & Method & $K$ & $q$ & Likelihood \\
\midrule
RT & GCNmf & 5 & 0 & legacy \\
RT & PEMix & 5 & 16 & gaussian \\
\midrule
U-MCAR & GCNmf & 5 & 0 & legacy \\
U-MCAR & PEMix & 5 & 8 & gaussian \\
\bottomrule
\end{tabular}
\end{table*}

\begin{table*}[t]
\centering
\caption{GMM settings selected for AQI. $K$ is the number of mixture components and $q$ is the PE dimension; GCNmf does not use PE ($q=0$).}
\label{tab:hyperparameters_aqi_gmm}
\small
\setlength{\tabcolsep}{7pt}
\begin{tabular}{lccl}
\toprule
Method & $K$ & $q$ & Likelihood \\
\midrule
GCNmf & 5 & 0 & legacy \\
PEMix & 5 & 8 & gaussian \\
\bottomrule
\end{tabular}
\end{table*}

\begin{table*}[t]
\centering
\caption{GMM settings selected for PV-US. $K$ is the number of mixture components and $q$ is the PE dimension; GCNmf does not use PE ($q=0$).}
\label{tab:hyperparameters_pv_us_gmm}
\small
\setlength{\tabcolsep}{7pt}
\begin{tabular}{lccl}
\toprule
Method & $K$ & $q$ & Likelihood \\
\midrule
GCNmf & 5 & 0 & gaussian \\
PEMix & 3 & 4 & gaussian \\
\bottomrule
\end{tabular}
\end{table*}

\begin{table*}[t]
\centering
\caption{GMM settings selected for METR-LA. $K$ is the number of mixture components and $q$ is the PE dimension; GCNmf does not use PE ($q=0$).}
\label{tab:hyperparameters_metr_la_gmm}
\small
\setlength{\tabcolsep}{7pt}
\begin{tabular}{lccl}
\toprule
Method & $K$ & $q$ & Likelihood \\
\midrule
GCNmf & 5 & 0 & gaussian \\
PEMix & 5 & 16 & gaussian \\
\bottomrule
\end{tabular}
\end{table*}

\clearpage

\subsection{Computational Complexity}
\label{app:complexity}

\paragraph{Notation and conventions.}
Let $n$ be the number of nodes, $m=\operatorname{nnz}(\hat{\mathbf A})\ge n$ the
number of non-zeros of the propagation operator (both edge directions and
self-loops), $d$ the input dimension, $h$ the hidden dimension, $c$ the output
dimension, $K$ the number of mixture components, $q$ the PE dimension, and $R$
the number of training epochs. All costs refer to a single input graph with
missingness mask $\mathbf M$. All models share the same backbone. For a
two-layer GCN that applies the linear map before propagation, one full-batch
forward (and, up to a constant, backward) pass costs
\begin{equation}
\mathcal C_{\mathrm B}(d)=O(ndh+mh+nhc+mc).
\end{equation}
We report separately (i) the cost of handling missing features, which for all
baselines depends on $\mathbf M$ and must be recomputed whenever the mask
changes, and (ii) the per-epoch cost of the learned model.

\paragraph{Imputation baselines.}
Zero, mean, and median imputation require one pass over the feature matrix,
$O(nd)$ (median via linear-time selection; sorting-based implementations cost
$O(nd\log n)$). GNN-MIM additionally concatenates the mask, so its first layer
acts on $2d$ inputs and costs $O(2ndh+mh)$; only the feature-transform term
doubles. FP performs $T_{\mathrm{FP}}$ sparse propagations, each followed by
re-injection of the observed values, for $O(T_{\mathrm{FP}}\,dm)$ time and
$O(m+nd)$ memory. PCFI additionally computes channel-wise shortest-path
distances to observed nodes ($O(dm)$ via multi-source BFS), diffuses each
channel with a channel-specific confidence-weighted operator
($O(T_{\mathrm{PC}}\,dm)$ time and $O(dm)$ memory when the $d$ operators are
materialized), and applies one inter-channel step based on the $d\times d$
channel correlation matrix, which costs $O(nd^2)$.
FISF runs a pre-diffusion on all channels ($T_1$ iterations), selects the
$\rho d$ channels with the lowest variance, and re-diffuses them ($T_2$
iterations) after recomputing their distances, giving
$O((T_1+\rho T_2)\,dm)$.
After imputation, training costs $R\,\mathcal C_{\mathrm B}(d)$.

\paragraph{Mixture-based models.}
GCNmf and \MODEL{} handle missing features inside every forward pass. For each
component $k$, the first layer propagates the component-wise mean and variance
matrices through $\hat{\mathbf A}$ and $\hat{\mathbf A}\odot\hat{\mathbf A}$,
evaluates the expected ReLU, and mixes the results with the responsibilities
$\gamma_{ik}$. The per-epoch cost is
\begin{equation}
\mathcal C_{\mathrm{mix}}
=O\big(K(ndh+mh)+nK(d+q)\big)+O(nhc+mc)
=O\big(K(ndh+mh)+nKq\big)+O(nhc+mc),
\end{equation}
with $q=0$ for GCNmf. The positional encodings enter only the responsibilities
and are never propagated, so $\mathbf W^{(1)}\in\mathbb R^{d\times h}$ and the
message-passing cost is identical to GCNmf. The additional term $O(nKq)$ is
dominated by $O(Kndh)$ whenever $q\le dh$, which holds in all our experiments;
\MODEL{} therefore has the same per-epoch complexity as GCNmf. Both are
approximately $K$ times more expensive than the backbone's first layer and
store $O(Knh)$ first-layer activations for backpropagation. Mixture parameters
are initialized by $I$ EM iterations on $[\tilde{\mathbf X}\Vert\mathbf P]$,
costing $O(I\,nK(d+q))$ for diagonal covariances.

\paragraph{Positional encodings.}
\MODEL{} computes the $q{+}1$ smallest eigenpairs of the normalized Laplacian
with a sparse iterative eigensolver. Each Lanczos/LOBPCG iteration costs
$O(mq+nq^2)$, while the number of iterations depends on the tolerance and on
the spectral gap, so we do not state a closed-form bound. This computation
depends only on $\mathbf A$: it is performed once per graph and shared across
all missingness rates, seeds, and masks. In contrast, the preprocessing of the
baselines is recomputed for every mask.

\begin{table}[ht]
\centering
\caption{Computational cost per input graph. Mask-dependent costs are
recomputed for every missingness pattern. $\dagger$: $O(nd\log n)$ with
sorting. $\ddagger$: implementation-dependent (see text).}
\label{tab:complexity}
\scriptsize\setlength{\tabcolsep}{2pt}
\begin{tabular}{llll}
\toprule
Method & Mask-dependent one-off & Per epoch & Extra memory \\
\midrule
GNN-Zero / Mean & $O(nd)$ & $\mathcal C_{\mathrm B}(d)$ & -- \\
GNN-Median & $O(nd)^{\dagger}$ & $\mathcal C_{\mathrm B}(d)$ & -- \\
GNN-MIM & $O(nd)$ & $\mathcal C_{\mathrm B}(d)+O(ndh)$ & $O(nd)$ \\
FP & $O(T_{\mathrm{FP}}\,dm)$ & $\mathcal C_{\mathrm B}(d)$ & $O(nd)$ \\
PCFI & $O(T_{\mathrm{PC}}\,dm+nd^2)^{\ddagger}$ & $\mathcal C_{\mathrm B}(d)$ & $O(dm+d^2)^{\ddagger}$ \\
FISF & $O((T_1+\rho T_2)\,dm)^{\ddagger}$ & $\mathcal C_{\mathrm B}(d)$ & $O(dm)^{\ddagger}$ \\
GCNmf & $O(I\,nKd)$ & $O(K(ndh+mh))+O(nhc+mc)$ & $O(Knh)$ \\
\MODEL{} & $O(I\,nK(d+q))$ & $O(K(ndh+mh)+nKq)+O(nhc+mc)$ & $O(Knh+nq)$ \\
\bottomrule
\end{tabular}
\end{table}
\noindent \MODEL{} additionally requires one sparse partial eigendecomposition
per graph, independent of the mask.

\paragraph{Inference.}
For a fixed mask, imputation-based baselines reuse the completed matrix and
require one backbone forward pass, whereas GCNmf and \MODEL{} require one
mixture forward pass. When the mask changes at test time, the baselines must
re-run their imputation, while the mixture models handle the new mask within
the forward pass.

\paragraph{Summary.}
Structural conditioning comes at a small cost: \MODEL{} has the same per-epoch
complexity as GCNmf, adding only a dominated $O(nKq)$ term to the mixture
responsibilities and a single sparse partial eigendecomposition per graph.
Compared with imputation-based baselines, both mixture models incur a moderate
constant-factor overhead per epoch, proportional to the small number of mixture components \(K\) used in our experiments”,
because missing features are handled inside the forward pass. In exchange, they
require no mask-dependent preprocessing, avoid the $O(nd^2)$ inter-channel cost
and $O(dm)$ channel-specific operators of PCFI and FISF, and handle new
missingness patterns without re-running imputation.

\section{Ablations}\label{app:ablations}\label{app:fusion-comparison}

This appendix complements RQ3 (\Cref{tab:rq3_pe}). We report a matched ablation that removes
PE conditioning from PEMix, PE-only controls, the effect of jointly optimising the mixture,
and the complete per-level results of the positional-encoding ablation on TADPOLE. We additionally include an MLP PE-only control to separate PE information from message passing

\subsection{Matched structural-conditioning ablation}\label{app:pe-ablation}

Most models in \Cref{tab:rq3_pe} differ from PEMix not only in how they use PE but also in
their architecture and optimisation settings. To isolate the contribution of PE conditioning
alone, we compare PEMix with the same model in which the positional encodings are set to
zero, so that the mixture responsibilities receive no structural information. This variant is
a GCNmf trained exactly as PEMix. The two models share masks, splits, optimisation settings,
and the probabilistic first layer, and coincide exactly at $\mu=0$, where no feature is
missing. As shown in \Cref{fig:pe_ablation_appendix}, removing PE conditioning lowers the AUC
from 0.8243 to 0.7330 on TADPOLE where the gap is
largest at high missingness: from $\mu=0.6$ onwards it grows steadily, reaching 0.49.

\begin{figure}[ht]
    \centering
    \includegraphics[width=0.50\textwidth]{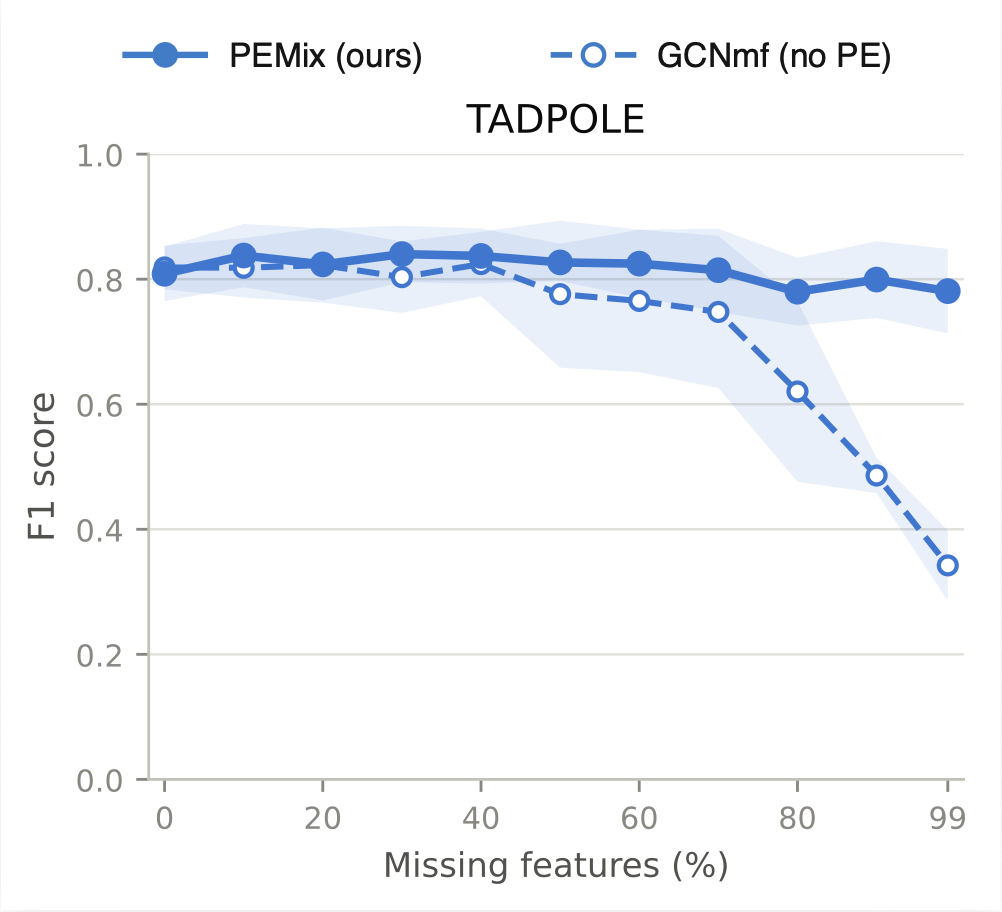}
    \caption{Matched ablation between PEMix and the same probabilistic recovery model without
    PE conditioning, obtained by setting the positional encodings to zero. The models share
    masks, splits, optimization settings, and the probabilistic first layer; at $\mu=0$ they
    coincide exactly. Mean $\pm$ standard deviation over five seeds.}
    \label{fig:pe_ablation_appendix}
\end{figure}

\subsection{PE-only controls}\label{app:peonly}

On TADPOLE, the PE-only row in \Cref{tab:rq3_pe} is a two-layer GCN whose only node input is
the PE, with the same depth, hidden width, dropout, and learning rate as PEMix. To separate
the contribution of PE from that of message passing, we also train an MLP with the same
depth, width, dropout, and learning rate on the PE alone, which ignores the graph. Since
neither control receives the node features, its score does not depend on $\mu$. The GCN
control obtains $0.295\pm0.092$ macro-F1 and the MLP control $0.254\pm0.027$, both well below
PEMix at every missingness level, including $\mu=0.99$ ($0.822\pm0.047$).

\subsection{Effect of joint GMM optimisation}
\label{app:gmm_joint}

Following \citet{TAGUCHI2021155}, we initialise the mixture parameters with EM and
subsequently optimise them jointly with the downstream prediction objective. To assess the
effect of this design choice, we compare the full model against a \emph{frozen-GMM} variant,
in which the mixture means, variances, and weights are fixed after EM initialisation, while
the GNN and prediction head remain trainable.

We evaluate both variants on EngRAD under \RTMAR{} at missing rates of
$20\%$, $50\%$, and $80\%$, using five seeds. As shown in
Table~\ref{tab:gmm_joint_frozen}, jointly optimising the mixture consistently
improves downstream forecasting performance compared with keeping the
EM-initialised mixture fixed.

\begin{table}[ht]
\centering
\caption{Effect of jointly optimising the GMM on EngRAD under \RTMAR{}.
Values are mean $\pm$ standard deviation over five seeds. Lower is better.}
\label{tab:gmm_joint_frozen}
\small\setlength{\tabcolsep}{12pt}\renewcommand{\arraystretch}{1.12}
\begin{tabular}{c cc}
\toprule
Missing rate & Joint GMM & Frozen GMM \\
\midrule
$20\%$ & $\mathbf{6.1751 \pm 0.0635}$ & $6.287 \pm 0.086$ \\
$50\%$ & $\mathbf{7.1536 \pm 0.0951}$ & $7.364 \pm 0.113$ \\
$80\%$ & $\mathbf{8.3722 \pm 0.0844}$ & $8.523 \pm 0.080$ \\
\bottomrule
\end{tabular}
\end{table}

\subsection{Positional-encoding ablation: complete results on TADPOLE}\label{app:pe-ablation-levels}
\Cref{tab:pe-ablation-tadpole,tab:pe-ablation-tadpole-high} report, for every model in
\Cref{tab:rq3_pe}, TADPOLE macro-F1 (mean $\pm$ standard deviation over five seeds) at each
missingness rate under RT. The last column gives the AUC reported in \Cref{tab:rq3_pe}.
GCNmf + PE (downstream) is a latent-fusion variant: the mixture is fitted on the node
features alone, and PE are concatenated to the representation produced by the probabilistic
first layer, so they cannot affect the mixture responsibilities.

\begin{table*}[p]
\centering
\caption{TADPOLE: positional-encoding ablation under RT, macro-F1 (mean $\pm$ standard deviation over five seeds). ``Mechanism'' uses PE while handling missing features; ``Downstream'' concatenates PE after missing-feature handling. Missingness $\mu$ is in percent; rates 0--50\%; AUC is reported in the continuation.}
\label{tab:pe-ablation-tadpole}
\footnotesize\setlength{\tabcolsep}{4pt}
\begin{adjustbox}{max width=\linewidth}
\begin{tabular}{llcccccc}
\toprule
PE usage & Model & 0 & 10 & 20 & 30 & 40 & 50 \\
\midrule
\multirow{4}{*}{Mechanism} & PEMix & $0.828\pm0.042$ & $0.864\pm0.041$ & $0.846\pm0.050$ & $0.838\pm0.050$ & $0.844\pm0.054$ & $0.826\pm0.028$ \\
 & FP + PE & $0.806\pm0.065$ & $0.818\pm0.024$ & $0.805\pm0.036$ & $0.803\pm0.030$ & $0.812\pm0.055$ & $0.776\pm0.074$ \\
 & PCFI + PE & $0.796\pm0.054$ & $0.790\pm0.047$ & $0.814\pm0.036$ & $0.817\pm0.040$ & $0.812\pm0.041$ & $0.818\pm0.026$ \\
 & FISF + PE & $0.806\pm0.040$ & $0.812\pm0.026$ & $0.806\pm0.042$ & $0.815\pm0.047$ & $0.813\pm0.041$ & $0.803\pm0.048$ \\
\midrule
\multirow{4}{*}{Downstream} & GCNmf + PE & $0.832\pm0.055$ & $0.819\pm0.053$ & $0.810\pm0.050$ & $0.812\pm0.046$ & $0.854\pm0.035$ & $0.816\pm0.037$ \\
 & FP + PE & $0.842\pm0.017$ & $0.824\pm0.026$ & $0.831\pm0.030$ & $0.785\pm0.047$ & $0.805\pm0.051$ & $0.780\pm0.054$ \\
 & PCFI + PE & $0.841\pm0.024$ & $0.809\pm0.032$ & $0.815\pm0.034$ & $0.811\pm0.017$ & $0.807\pm0.036$ & $0.824\pm0.038$ \\
 & FISF + PE & $0.835\pm0.016$ & $0.816\pm0.016$ & $0.810\pm0.055$ & $0.810\pm0.042$ & $0.845\pm0.033$ & $0.765\pm0.063$ \\
\midrule
-- & PE only & $0.295\pm0.092$ & $0.295\pm0.092$ & $0.295\pm0.092$ & $0.295\pm0.092$ & $0.295\pm0.092$ & $0.295\pm0.092$ \\
\bottomrule
\end{tabular}
\end{adjustbox}
\end{table*}

\begin{table*}[p]
\centering
\caption{TADPOLE: positional-encoding ablation under RT, macro-F1 (mean $\pm$ standard deviation over five seeds). ``Mechanism'' uses PE while handling missing features; ``Downstream'' concatenates PE after missing-feature handling. Missingness $\mu$ is in percent; rates 60--99\% and trapezoidal AUC over $\mu\in[0,0.99]$ (higher is better).}
\label{tab:pe-ablation-tadpole-high}
\footnotesize\setlength{\tabcolsep}{4pt}
\begin{adjustbox}{max width=\linewidth}
\begin{tabular}{llcccccc}
\toprule
PE usage & Model & 60 & 70 & 80 & 90 & 99 & AUC \\
\midrule
\multirow{4}{*}{Mechanism} & PEMix & $0.839\pm0.045$ & $0.844\pm0.039$ & $0.808\pm0.043$ & $0.791\pm0.067$ & $0.822\pm0.047$ & \textbf{0.8243} \\
 & FP + PE & $0.792\pm0.067$ & $0.722\pm0.146$ & $0.685\pm0.125$ & $0.562\pm0.082$ & $0.412\pm0.075$ & 0.7336 \\
 & PCFI + PE & $0.818\pm0.029$ & $0.796\pm0.061$ & $0.670\pm0.119$ & $0.559\pm0.114$ & $0.464\pm0.076$ & 0.7474 \\
 & FISF + PE & $0.822\pm0.046$ & $0.777\pm0.102$ & $0.701\pm0.041$ & $0.633\pm0.105$ & $0.491\pm0.057$ & 0.7573 \\
\midrule
\multirow{4}{*}{Downstream} & GCNmf + PE & $0.755\pm0.079$ & $0.734\pm0.163$ & $0.577\pm0.131$ & $0.490\pm0.035$ & $0.339\pm0.073$ & 0.7211 \\
 & FP + PE & $0.782\pm0.078$ & $0.775\pm0.048$ & $0.744\pm0.133$ & $0.762\pm0.054$ & $0.631\pm0.114$ & 0.7754 \\
 & PCFI + PE & $0.763\pm0.099$ & $0.807\pm0.063$ & $0.719\pm0.126$ & $0.736\pm0.095$ & $0.600\pm0.069$ & 0.7745 \\
 & FISF + PE & $0.825\pm0.052$ & $0.807\pm0.053$ & $0.731\pm0.153$ & $0.737\pm0.157$ & $0.603\pm0.092$ & 0.7798 \\
\midrule
-- & PE only & $0.295\pm0.092$ & $0.295\pm0.092$ & $0.295\pm0.092$ & $0.295\pm0.092$ & $0.295\pm0.092$ & 0.2923 \\
\bottomrule
\end{tabular}
\end{adjustbox}
\end{table*}

\end{document}